\documentclass{article}

\usepackage[margin = 1in]{geometry}

\usepackage{mathpazo}

\usepackage{algorithm, caption}
\usepackage[noend]{algpseudocode}
\usepackage{listings}

\usepackage{enumitem}
\usepackage{amsmath}
\usepackage{amssymb}
\usepackage{graphicx}
\usepackage{amsthm}
\usepackage{amssymb}
\usepackage{hyperref} 
\usepackage{color}
\usepackage{calrsfs}
\usepackage{thmtools}
\usepackage{bbm}

\newcommand\restr[2]{{
  \left.\kern-\nulldelimiterspace 
  #1 
  \vphantom{\big|} 
  \right|_{#2} 
  }}

\newtheorem{theorem}{Theorem}[section]
\newtheorem{lemma}[theorem]{Lemma}

\newtheorem{proposition}[theorem]{Proposition} 
\newtheorem{corollary}[theorem]{Corollary} 

\theoremstyle{definition}
\newtheorem{definition}[theorem]{Definition}

\newtheorem{assumption}{Assumption}

\DeclareMathOperator{\Var}{Var}
\newcommand{\norm}[1]{\left \lVert #1 \right \rVert}

\newenvironment{namedproof}[1]{\paragraph{Proof of #1.}\hspace{-1em}}{\hfill\qed}

\DeclareMathOperator*{\argmin}{argmin}

\newcommand{\RR}{\mathbb{R}}
\DeclareMathOperator{\supp}{supp}
\DeclareMathOperator{\nll}{nll}
\newcommand{\bx}{\breve{x}}
\newcommand{\bz}{\breve{z}}

\newcommand{\smax}{s_{\max{}}}
\DeclareMathOperator{\sign}{sign}
\DeclareMathOperator{\poly}{poly}

\title{Truncated Linear Regression in High Dimensions}

\author{Constantinos Daskalakis \\ MIT \\ \texttt{costis@mit.edu} \and Dhruv Rohatgi \\ MIT \\ \texttt{drohatgi@mit.edu} \and Manolis Zampetakis \\ MIT \\ \texttt{mzampet@mit.edu}}

\begin{document}

\let\endtitlepage\relax
\begin{titlepage}
\maketitle
\end{titlepage}

\begin{abstract}
    As in standard linear regression, in {\em truncated} linear regression, we are given
  access to observations $(A_i, y_i)_i$ whose dependent variable equals 
  $y_i= A_i^{\rm T} \cdot x^* + \eta_i$, where $x^*$ is some fixed unknown vector of
  interest and $\eta_i$ is independent noise; except we are only given an observation if its dependent
  variable $y_i$ lies in some ``truncation set'' $S \subset \mathbb{R}$. The goal is to
  recover~$x^*$ under some favorable conditions on the $A_i$'s and the noise
  distribution. We prove that there exists a computationally and statistically efficient
  method for recovering $k$-sparse $n$-dimensional vectors $x^*$ from $m$ truncated samples, which attains
  an optimal $\ell_2$ reconstruction error of $O(\sqrt{(k \log n)/m})$. As a corollary, our guarantees imply a
  computationally efficient and information-theoretically optimal algorithm for compressed 
  sensing with truncation, which may arise from measurement saturation effects.
  Our result follows from a statistical and computational analysis of the Stochastic
  Gradient Descent (SGD) algorithm for solving a natural adaptation of the LASSO optimization problem that accommodates truncation. This generalizes the works of both: (1)
  Daskalakis et al.~\cite{Daskalakis2018}, where no regularization is needed due to
  the low-dimensionality of the data, and (2) Wainright~\cite{Wainwright2009}, where
  the objective function is simple due to the absence of truncation. In order to deal
  with both truncation and high-dimensionality at the same time, we develop new
  techniques that not only generalize the existing ones but we believe are of
  independent interest.
  
\end{abstract}


\section{Introduction}

  In the vanilla linear regression setting, we are given $m \ge n$ observations of the
 form $(A_i, y_i)$, where $A_i \in \mathbb{R}^n$, $y_i = A_i^{\rm T} x^* + \eta_i$,  
 $x^*$ is some unknown coefficient vector that we wish to recover, and $\eta_i$ is independent and identically distributed across different observations $i$ random noise. Under favorable conditions about the $A_i$'s and the distribution of the noise, it is well-known
 that $x^*$ can be recovered to within $\ell_2$-reconstruction error  $O(\sqrt{n/m})$. 
 

The classical model and its associated guarantees might, however, be inadequate to address many situations which frequently  arise in both theory and practice. We focus on two common and widely studied deviations from the
standard model. First, it is often the case that $m \ll n$, i.e.~the number of observations is
much smaller than the dimension of the unknown vector $x^*$. In this ``under-determined''
regime, it is fairly clear that it is impossible to expect a non-trivial reconstruction of the underlying $x^*$, since there are infinitely many 
$x \in \mathbb{R}^n$ such that $A_i^{\rm T}x=A_i^{\rm T}x^*$ for all $i=1,\ldots,m$. To
sidestep this impossibility, we must exploit additional structural properties that we might know  $x^*$ satisfies.
One such property might be {\em sparsity}, i.e.~that $x^*$ has $k \ll n$ non-zero coordinates. Linear regression under sparsity assumptions has been widely studied, motivated by  applications such as model selection and
compressed sensing; see e.g.~the celebrated works of~\cite{tibshirani1996regression,Candes2006,donoho2006most,Wainwright2009} on this topic. It is known, in particular, that a $k$-sparse $x^*$ can be recovered to within $\ell_2$ error $O(\sqrt{k \log n/m})$, when the $A_i$'s are drawn from the standard multivariate
Normal, or satisfy other favorable conditions~\cite{Wainwright2009}. The recovery algorithm solves a least squares optimization problem with $\ell_1$ regularization, i.e.~what is called LASSO optimization in Statistics, in order to reward sparsity.

Another common deviation from the standard model is the presence of {\em truncation}. 
Truncation occurs when the sample $(A_i, y_i)$ is not observed whenever $y_i$ falls outside of a 
subset $S \subseteq \mathbb{R}$. Truncation arises quite often in practice as a result of saturation of measurement devices, bad data collection practices, incorrect experimental design, and legal or privacy constraints which might preclude the use of some of the data. Truncation is known to affect linear regression in counter-intuitive ways, as illustrated in Fig.~\ref{fig:illustration},
\begin{figure}[h!]
    \begin{center}
    \includegraphics[width=0.6\textwidth]{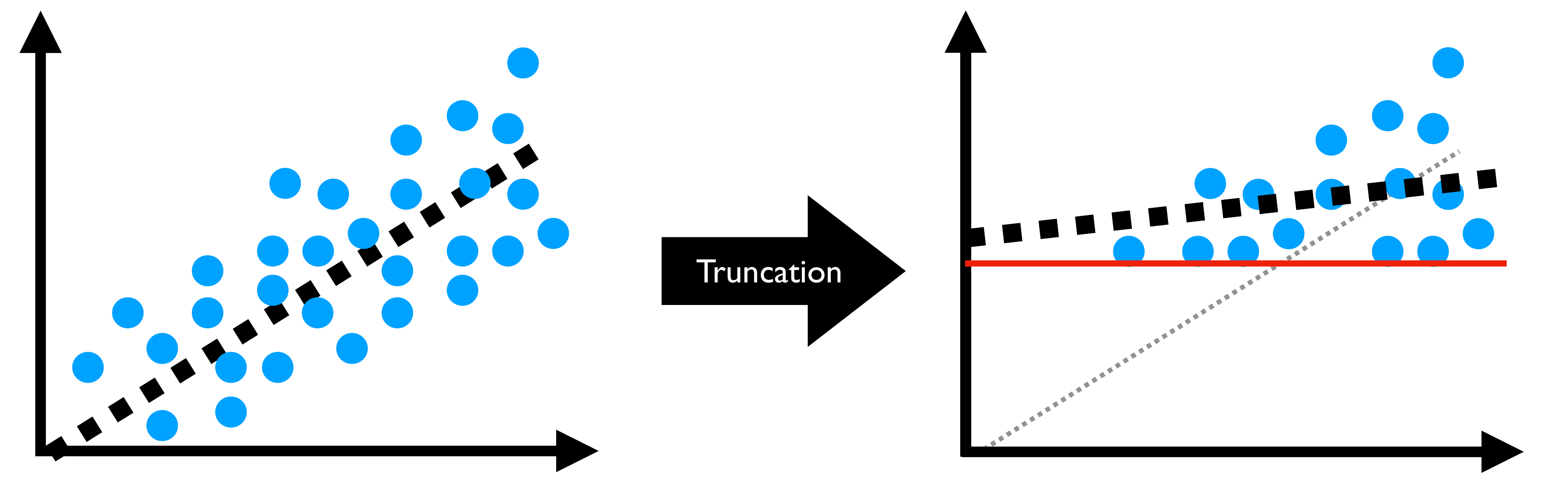}
    \caption{Truncation in one-dimensional linear regression, along with the linear fit obtained via least squares regression before and after truncation.}
    \label{fig:illustration}
    \end{center}
\end{figure}
where the linear fits obtained via least squares regression before and after truncation of the data based on the value of the response variable are
also shown. More broadly, it is well-understood that naive statistical inference using truncated
data commonly leads to  bias. Accordingly, a long line of research in Statistics and
Econometrics has strived to develop regression methods that are robust to
truncation~\cite{tobin1958estimation,amemiya1973regression,hausman1977social,maddala1986limited,keane199320,breen1996regression,hajivassiliou1998method}. 
This line of work falls into the broader field of Truncated
Statistics~\cite{Schneider86,Cohen91,BalakrishnanCramer}, which finds its roots in the early
works of~\cite{bernoulli1760essai}, \cite{Galton1897}, \cite{Pearson1902,PearsonLee1908},
and~\cite{fisher31}. Despite this voluminous work, computationally and statistically efficient methods for
truncated linear regression have only recently been obtained in~\cite{Daskalakis2018}, where it was shown that, under  favorable assumptions about the $A_i$'s,
the truncation set $S$, and assuming the $\eta_i$'s are drawn from a Gaussian, the negative log
likelihood of the truncated sample can be optimized efficiently, and approximately recovers the
true parameter vector with an $\ell_2$ reconstruction error ${O}\left(\sqrt{n \log m \over m}\right)$. 

\paragraph{Our contribution.} In this work, we solve the general problem addressing
both of the aforedescribed challenges together. Namely, we provide efficient 
algorithms for the high-dimensional ($m\ll n$) truncated linear regression problem. This 
problem is very common, including in compressed sensing applications with measurement saturation,
as studied e.g.~in ~\cite{davenport2009simple,laska2010democracy}.

Under standard  conditions on the design matrix and the noise distribution (namely that the $A_i$'s and $\eta_i$'s are sampled from independent Gaussian distributions before truncation), and under mild assumptions on the truncation set $S$ (roughly, that it permits a constant fraction of the samples $y_i = A_i^{\rm T}x^*+\eta_i$ to survive truncation), we show that the
SGD algorithm on the {\em truncated LASSO} optimization program, our proposed adaptation of the standard LASSO optimization to accommodate truncation, is a computationally and
statistically efficient method for recovering $x^*$, attaining an optimal $\ell_2$ reconstruction error of $O(\sqrt{(k\log n)/m})$, where $k$ is the sparsity of $x^*$.

We formally state the model and assumptions in Section~\ref{section:model}, and our result in Section~\ref{sec:results}.

\subsection{Overview of proofs and techniques} \label{sec:intro:techniques}

The problem that we solve in this paper encompasses the two difficulties of the problems considered in: 
(1) Wainwright \cite{Wainwright2009}, which tackles the problem of high-dimensional sparse linear 
regression with Gaussian noise, and (2) Daskalakis et al. \cite{Daskalakis2018}, which tackles 
the problem of truncated linear regression. The tools developed in those papers do not suffice to solve
our problem, since each difficulty interferes with the other. Hence, we introduce new ideas and develop 
new interesting tools that allow us to bridge the gap between \cite{Wainwright2009} 
and \cite{Daskalakis2018}. We begin our overview in this section with a brief description of the
approaches of \cite{Wainwright2009, Daskalakis2018} and subsequently outline the additional
challenges that arise in our setting, and how we address them.

Wainwright \cite{Wainwright2009} uses as an estimator the solution of the regularized least 
squares program, also called \textit{LASSO}, to handle the high-dimensionality of the data. Wainwright
then uses a primal-dual witness method to bound the number of samples that are needed in
order for the solution of the LASSO program to be close to the true coefficient vector $x^*$.
The computational task is not discussed in detail in Wainwright \cite{Wainwright2009}, since the
objective function of the LASSO program is very simple and standard convex optimization tools
can be used.

Daskalakis et al. \cite{Daskalakis2018} use as their estimator the solution to
the log-likelihood maximization problem. In contrast to \cite{Wainwright2009}, their convex 
optimization problem takes a very complicated form due to the presence of truncation, which introduces an intractable log-partition function term in the log-likelihood. The 
main idea of Daskalakis et al. \cite{Daskalakis2018} to overcome this difficulty is identifying a 
convex set $D$ such that: (1) it contains the true coefficient vector $x^*$, (2) their
objective function is strongly convex inside $D$, (3) inside $D$ there exists an efficient
rejection sampling algorithm to compute unbiased estimates of the gradients of their objective 
function, (4) the norm of the stochastic gradients inside $D$ is bounded, and (5) projecting onto $D$ is 
efficient. These five properties are essentially what they need to prove that the SGD with projection set
$D$ converges quickly to a good estimate of $x^*$.

Our reconstruction algorithm is inspired by both \cite{Wainwright2009} and \cite{Daskalakis2018}. We
formulate our optimization program as the $\ell_1$-regularized version of the negative
log-likelihood function in the truncated setting, which we call \textit{truncated LASSO}. In particular, our objective contains an intractable log-partition function term. Our
proof then consists of two parts. First, we show statistical recovery, i.e. we upper bound
the number of samples that are needed for the solution of the truncated LASSO program to be
close to the true coefficient vector $x^*$. Second, we show that this optimization problem can be
solved efficiently. The cornerstones of our proof are the two seminal approaches that we mentioned:
the Primal-Dual Witness method for statistical recovery in high dimensions, and the Projected
SGD method for efficient maximum likelihood estimation in the presence of truncation. Unfortunately, these two
techniques are not a priori compatible to stitch together.

Roughly speaking, the technique of \cite{Daskalakis2018} relies heavily on the very carefully
chosen \emph{projection set} $D$ in which the SGD is restricted, as we explained above. This
projection set cannot be used in high dimensions because it effectively requires knowing the
low-dimension subspace in which the true solution lies. The projection set was the key to the 
aforementioned nice properties: strong convexity, efficient gradient estimation, and bounded gradient. In
its absence, we need to deal with each of these issues individually. The primal-dual witness method of 
\cite{Wainwright2009} cannot also be applied directly in our setting. In our case the gradient of the 
truncated LASSO does not have a nice closed form and hence finding the correct way to construct the 
primal-dual witness requires a more delicate argument. Our proof manages to overcome all these 
issues. In a nutshell the architecture of our full proof is the following.

\begin{enumerate}
    \item \textbf{Optimality on the low dimensional subspace.} The first thing that we need to prove
    is that the optimum of the truncated LASSO program when restricted to the low dimensional
    subspace defined by the non-zero coordinates of $x^*$ is close to the true solution. This step of
    the proof was unnecessary in \cite{Daskalakis2018} due to the lack of regularization in their 
    objective, and was trivial in \cite{Wainwright2009} due to the simple loss function, 
    i.e. the regularized least square.
    \item \textbf{Optimality on the high dimensional space.} We prove that the optimum of the 
    truncated LASSO program in the low dimensional subspace is also optimal for the whole space. This 
    step is done using the primal-dual witness method as in \cite{Wainwright2009}. However, in
    our case the expression of the gradient is much more complicated due to the very convoluted
    objective function. Hence, we find a more general way to prove this step that does not rely
    on the exact expression of the gradient.
\end{enumerate}

These two steps of the proof suffice to upper bound the number of samples that we need to recover the 
coefficient vector $x^*$ via the truncated LASSO program. Next, we provide a computationally
efficient method to solve the truncated LASSO program.

\begin{enumerate}
    \setcounter{enumi}{2}
    \item \textbf{Initialization of SGD.} The first step of our algorithm to solve truncated LASSO is 
    finding a good initial point for the SGD. This was unnecessary in \cite{Wainwright2009} due to the
    simple objective and in \cite{Daskalakis2018} due to the existence of the projection set $D$ (where 
    efficient projection onto $D$ immediately gave an initial point). We propose the simple
    answer of \emph{bootstrapping}: start with the solution of the $\ell_1$-regularized ordinary least 
    squares program. This is a biased estimate, but we show it's good enough for initialization.
    \item \textbf{Projection of SGD.} Next, we need to choose a projection set to make sure that Projected-SGD (PSGD)
    converges. The projection set chosen in \cite{Daskalakis2018} is not helpful in our case unless we
    a priori know the set of non-zero coordinates of $x^*$. Hence, we define a different, simpler 
    set which admits efficient projection algorithms. As a necessary side-effect, in contrast to 
    \cite{Daskalakis2018}, our set cannot guarantee many of the important properties that we need to 
    prove fast convergence of SGD.
    \item \textbf{Lack of strong convexity and gradient estimation}. Our different projection set cannot 
    guarantee the strong convexity and efficient gradient estimation enjoyed in \cite{Daskalakis2018}. 
    There are two problems here:
    
    First, we know that PSGD converges to a point with small loss, but why must the point be near the
    optimum? Since strong convexity fails in high dimensions, it is not clear. We provide a
    workaround to resolve this issue that can be applied to other regularized programs with stochastic
    access to the gradient function.
    
    Second, computing unbiased estimates of the gradient is now difficult. The prior work employed 
    rejection sampling, but in our setting this may take exponential time. For this reason we provide a 
    more explicit method for estimating the gradient much faster, whenever the truncation set is 
    reasonably well-behaved.
\end{enumerate}

  An important tool that we leverage repeatedly in our analysis and we have not mentioned above is a
strong isometry property for our measurement matrix, which has truncated Gaussian rows. Similar properties
have been explored in the compressed sensing literature for matrices with i.i.d. Gaussian and sub-Gaussian
entries \cite{Voroninski2016}.
  
We refer to Section~\ref{section:techniques} for a more detailed overview of the proofs of our main
results.

\section{High-dimensional truncated linear regression model} \label{section:model}

\paragraph{Notation.} Let $Z \sim N(0, 1)$ refer to a standard normal random variable. For 
$t \in \mathbb{R}$ and measurable $S \subseteq \RR$, let $Z_t \sim N(t,1;S)$ refer to the truncated
normal $(Z + t) | Z + t \in S$. Let $\mu_t = \mathbb{E}[Z_t]$. Let $\gamma_S(t) = \Pr[Z+t \in S]$.
Additionally, for $a,x \in \RR^n$ let $Z_{a,x}$ refer to $Z_{a^T x}$ (or $Z_{ax}$ if $a$ is a row
vector), and let $\gamma_S(a,x)$ refer to $\gamma_S(a^T x)$. For a matrix 
$A \in \mathbb{R}^{m \times n}$, let $Z_{A,x} \in \mathbb{R}^m$ be the random vector with
$(Z_{A,x})_j = Z_{A_j,x}$. For sets
$I \subseteq [n]$ and $J \subseteq [m]$, let $A_{I,J}$ refer to the submatrix 
$[A_{i,j}]_{i \in I, j \in J}$. For $i \in [n]$ we treat the row $A_i$ as a row vector. In a slight
abuse of notation, we will often write $A_U$ (or sometimes, $A_V$); this will \emph{always} mean
$A_{[m],U}$. By $A_U^T$ we mean $(A_{[m],U})^T$). For $x \in \mathbb{R}^n$, define $\supp(x)$ to be
the set of indices $i \in [n]$ such that $x_i \neq 0$.

\subsection{Model}\label{subsection:model}

Let $x^* \in \mathbb{R}^n$ be the unknown parameter vector which we are trying to recover. We assume that it is $k$-sparse; that is, $\supp(x^*)$ has cardinality at most $k$. Let $S \subseteq \mathbb{R}$ be a measurable subset of the real line. The main focus of this paper is the setting of \emph{Gaussian noise}: we assume that we are given $m$ truncated samples $(A_i, y_i)$ generated by the following process:

\begin{enumerate}
    \item Pick $A_i \in \mathbb{R}^n$ according to the standard normal distribution $N(0,1)^n$.
    \item Sample $\eta_i \sim N(0,1)$ and compute $y_i$ as
    \begin{equation} y_i = A_ix^* + \eta_i.
\label{process}
    \end{equation}
    \item If $y_i \in S$, then return sample $(A_i, y_i)$. Otherwise restart the process from step $1$.
\end{enumerate}

We also briefly discuss the setting of \emph{arbitrary noise}, in which $\eta_i$ may be arbitrary and we are interested in approximations to $x^*$ which have guarantees bounded in terms of $\norm{\eta}_2$.

Together, $m$ samples define a pair $(A,y)$ where $A \in \mathbb{R}^{m \times n}$ and $y \in \mathbb{R}^m$. We make the following assumptions about set $S$.

\begin{assumption}[Constant Survival Probability] \label{assumption:csp}
    Taking expectation over vectors $a \sim N(0,1)^n$, we have 
  $\mathbb{E} \gamma_S(a,x^*) \geq \alpha$ for a constant $\alpha > 0$.
\end{assumption}

\begin{assumption}[Efficient Sampling] \label{assumption:sampling}
    There is an $T(\gamma_{S}(t))$-time algorithm which takes input 
  $t \in \mathbb{R}$ and produces an unbiased sample $z \sim N(t, 1; S)$. 
\end{assumption}

  We do not require that $T(\cdot)$ is a constant, but it will affect the
efficiency of our algorithm. To be precise, our algorithm will make
$\text{poly}(n)$ queries to the sampling algorithm.
As we explain in Lemma \ref{lem:samplingUnionOfIntervalsLemma} in Section~\ref{sec:unionIntervalsSampling}, if the set $S$ is a union of $r$ intervals
$\cup_{i = 1}^r [a_i, b_i]$, then the Assumption \ref{assumption:sampling}
is satisfied with $T(\gamma_S(t)) = \poly(\log(1/\gamma_S(t)), r)$. 
We express the theorems below with the assumption that $S$ is a union of $r$ 
intervals in which case the algorithms have polynomial running time, but all the statements below
can be replaced with the more general Assumption \ref{assumption:sampling} and the running time 
changes from $\poly(n, r)$ to $\poly(n, T(e^{m/\alpha}))$.

\section{Statistically and computationally efficient recovery} \label{sec:results}

  In this section we formally state our main results for recovery of a sparse high-dimensional 
coefficient vector from truncated linear regression samples. In Section \ref{sec:results:noisy}, we
present our result under the standard assumption that the error distribution is Gaussian, whereas in
Section \ref{sec:results:adversarial} we present our results for the case of adversarial error.

\subsection{Gaussian noise} \label{sec:results:noisy}

  In the setting of Gaussian noise (before truncation), we prove the
following theorem.

\begin{theorem}\label{thm:main-result}
    Suppose that Assumption~\ref{assumption:csp} holds, and that we have $m$ samples 
  $(A_i, y_i)$ generated from Process~\eqref{process}, with $n \geq m \geq Mk\log n$ for a
  sufficiently large constant $M$.
  Then, there is an algorithm which outputs $\bar{x}$ satisfying 
  $\norm{\bar{x} - x^*}_2 \leq O(\sqrt{(k\log n)/m})$ with probability $1 - O(1/\log n)$. 
  Furthermore, if the survival set $S$ is a union of $r$ intervals
  the running time of our algorithm is $\poly(n, r)$.
\end{theorem}

From now on, we will use the term ``with high probability'' when the rate of decay is not of importance. This phrase means ``with probability $1 - o(1)$ as $n \to \infty$''.

  Observe that even without the added difficulty of truncation 
(e.g. if $S = \mathbb{R}$), sparse linear regression requires 
$\Omega(k \log n)$ samples by known information-theoretic arguments 
\cite{Wainwright2009}. Thus, our sample complexity is 
information-theoretically optimal.

  In one sentence, the algorithm optimizes the $\ell_1$-regularized sample
negative log-likelihood via projected SGD. The negative log-likelihood of 
$x \in \RR^n$ for a single sample $(a,y)$ is 
\[\nll(x; a,y) = \frac{1}{2}(a^T x - y)^2 + \log \int_S e^{-(a^T x - z)^2/2} \, dz.\]

  Given multiple samples $(A,y)$, we can then write 
$\nll(x;A,y) = \frac{1}{m} \sum_{j=1}^m \nll(x;A_j,y_j)$. We also define the regularized negative
log-likelihood $f: \mathbb{R}^n \to \mathbb{R}$ by $f(x) = \nll(x;A,y) + \lambda \norm{x}_1$. We
claim that optimizing the following program approximately recovers the true parameter vector $x^*$
with high probability, for sufficiently many samples and appropriate regularization $\lambda$:
\begin{equation} \label{eq:program}
  \min_{x \in \mathbb{R}^n} \nll(x;A,y) + \lambda \norm{x}_1.  
\end{equation}
The first step is to show that any solution to Program \eqref{eq:program}
will be near the true solution $x^*$. To this end, we prove the following 
theorem, which already shows that $O(k\log n)$ samples are sufficient to 
solve the problem of \emph{statistical} recovery of $x^*$:

\begin{proposition} \label{prop:statistical}
    Suppose that Assumption~\ref{assumption:csp} holds. There are
  constants\footnote{In the entirety of this paper, constants may depend on
  $\alpha$.} $\kappa$, $d$, and $\sigma$ with the following property. 
  Suppose that $m > \kappa \cdot k \log n$, and let $(A,y)$ be $m$ samples drawn 
  from Process~\ref{process}. Let $\hat{x}$ be any optimal solution to 
  Program~\eqref{eq:program} with regularization constant 
  $\lambda = \sigma \sqrt{(\log n)/m}$. Then 
  $\norm{\hat{x} - x^*}_2 \leq d\sqrt{(k\log n)/m}$ with high
  probability.
\end{proposition}

  Then it remains to show that Program~\eqref{eq:program} can be solved
efficiently. 

\begin{proposition} \label{prop:algorithm}
    Suppose that Assumption~\ref{assumption:csp} holds and let $(A,y)$ be $m$ samples
  drawn from Process~\ref{process} and $\hat{x}$ be any optimal solution to
  Program~\eqref{eq:program}. There exists a constant $M$ such that if $m \geq Mk\log n$
  then there is an algorithm which outputs $\bar{x} \in \mathbb{R}^n$ satisfying 
  $\norm{\bar{x} - \hat{x}}_2 \leq O(\sqrt{(k\log n)/m})$ 
  with high probability. Furthermore, 
  if the survival set $S$ is a 
  union of $r$ intervals the running time of our algorithm is 
  $\poly(n, r)$.
\end{proposition}

  We present a more detailed description of the algorithm that we use in Section 
\ref{section:effalg}.


\subsection{Adversarial noise} \label{sec:results:adversarial}

In the setting of arbitrary noise, optimizing negative log-likelihood no longer makes sense, and indeed our results from the setting of Gaussian noise no longer hold. However, we may apply results from compressed sensing which describe sufficient conditions on the measurement matrix for recovery to be possible in the face of adversarial error. We obtain the following theorem:

\begin{theorem} \label{thm:mainAdversarial}
    Suppose that Assumption~\ref{assumption:csp} holds and let 
  $\epsilon > 0$. There are constants $c$ and $M$ such that if 
  $m \geq Mk\log n$, $\norm{Ax^* - y}_2 \leq \epsilon$, and 
  $\hat{x}$ minimizes $\norm{x}_1$ in the region $\{x \in \mathbb{R}^n: \norm{Ax - y}_2 \leq \epsilon\}$,
  then $\norm{\hat{x} - x^*}_2 \leq c\epsilon/\sqrt{m}$.
\end{theorem}

  The proof is a corollary of our result that $A$ satisfies the Restricted
Isometry Property from \cite{Candes2006} with high probability even when we
only observe truncated samples; see Corollary~\ref{corollary:rip} and the subsequent discussion in Section~\ref{appendix:isom}.

The remainder of the paper is dedicated to the case where the noise is Gaussian before truncation.

\section{The efficient estimation algorithm}\label{section:effalg}

Define $\mathcal{E}_r = \{x \in \RR^n: \norm{Ax-y}_2 \leq r\sqrt{m}\}$. To solve Program~\ref{eq:program}, our algorithm is Projected Stochastic Gradient Descent (PSGD) with projection set $\mathcal{E}_r$, for an appropriate constant $r$ (specified in Lemma~\ref{lemma:solution-feasible}). We pick an initial feasible point by
computing \[x^{(0)} = \argmin_{x \in \mathcal{E}_r} \norm{x}_1.\] 
Subsequently, the algorithm performs $N$ updates, where $N = \poly(n)$. Define a random update to $x^{(t)} \in \mathbb{R}^n$ as follows. Pick $j \in [m]$ uniformly at random. Sample $z^{(t)} \sim Z_{A_j, x^{(i)}}$. Then set \[v^{(t)} := A_j(z^{(t)} - y_j) + \lambda \cdot \sign(x^{(t)})\]
\[w^{(t)} := x^{(t)} - \sqrt{\frac{1}{nN}} v^{(t)}; \qquad x^{(t+1)} := \argmin_{x \in \mathcal{E}_r} \norm{x - w^{(t)}}_2.\]
Finally, the algorithm outputs $\bar{x} = \frac{1}{N} \sum_{t=0}^{N-1} x^{(t)}$.

See Section~\ref{section:techniques:computational} for the motivation of this algorithm, and a proof sketch of correctness and efficiency. Section~\ref{sec:algos} contains a summary of the complete algorithm in pseudocode.

\section{Overview of proofs and techniques}\label{section:techniques}

   This section outlines our techniques. The first step is proving Proposition~\ref{prop:statistical}. The second step is proving Proposition~\ref{prop:algorithm}, by showing that the algorithm described in Section~\ref{section:effalg} efficiently recovers an approximate solution to Program~\ref{eq:program}.

\subsection{Statistical recovery}

  Our approach to proving Proposition \ref{prop:statistical} is the
Primal-Dual Witness (PDW) method introduced in \cite{Wainwright2009}. 
Namely, we are interested in showing that the solution of Program \eqref{eq:program}
is near $x^*$ with high probability. 
Let $U$ be the (unknown) support of the true parameter vector $x^*$. Define the following (hypothetical) program in which the solution space is restricted to vectors with support 
in $U$:
\begin{equation} \argmin_{x \in \RR^n: \supp(x) \subseteq U} \nll(x;A,y) + \lambda \norm{x}_1 \label{eq:res-program}\end{equation}
In the untruncated setting, the PDW method is to apply the subgradient optimality condition to the solution $\bx$ of this restricted program, which is by definition sparse. Proving that $\bx$ satisfies subgradient optimality for the original program implies that the original program has a sparse solution $\bx$, and under mild extra conditions $\bx$ is the unique solution. Thus, the original program recovers the true basis. We use the PDW method for a slightly different purpose; we apply subgradient optimality to $\bx$ to show that $\norm{\bx - x^*}_2$ is small, and then use this to prove that $\bx$ solves the original program.

Truncation introduces its own challenges. While Program~\ref{eq:program} is still convex \cite{Daskalakis2018}, it is much more convoluted than ordinary least squares. In particular, the gradient and Hessian of the negative log-likelihood have the following form (see Section~\ref{appendix:pdw} for the proof).

\begin{lemma}\label{lemma:gradient}
For all $(A,y)$, the gradient of the negative log-likelihood is $\nabla \nll(x;A,y) = \frac{1}{m} \sum_{j=1}^m A_j^T(\mathbb{E} Z_{A_j, x} - y_j).$
The Hessian is $H(x;A,y) = \frac{1}{m} \sum_{j=1}^m A_j^T A_j \Var(Z_{A_j, x}).$
\end{lemma}

We now state the two key facts which make the PDW method work. First, the solution $\bx$ to the restricted program must have a zero subgradient in all directions in $\mathbb{R}^U$. Second, if this subgradient can be extended to all of $\mathbb{R}^n$, then $\bx$ is optimal for the original program. Formally:

\begin{lemma}\label{lemma:pdw}
Fix any $(A,y)$. Let $\bx$ be an optimal solution to Program~\ref{eq:res-program}.
\begin{enumerate}[label  =(\alph*)]
\item There is some $\bz_U \in \mathbb{R}^U$ such that $\norm{\bz_U}_\infty \leq 1$ and $-\lambda \bz_U = \frac{1}{m} A_U^T(\mathbb{E} Z_{A,\bx} - y).$
\item Extend $\bz_U$ to $\mathbb{R}^n$ by defining $-\lambda \bz_{U^c} = \frac{1}{m} A_{U^c}^T (\mathbb{E} Z_{A,\bx} - y).$
If $\norm{\bz_{U^c}}_\infty < 1$, and $A_U^T A_U$ is invertible, then $\bx$ is the unique optimal solution to Program~\ref{eq:program}.
\end{enumerate}
\end{lemma}

See Section~\ref{appendix:pdw} for the proof. The utility of this lemma is in reducing Proposition~\ref{prop:statistical} to showing that with high probability over $(A,y)$, the following conditions both hold:
\begin{enumerate}
    \item $\norm{\bx - x^*}$ is small (Theorem~\ref{thm:dist}).
    \item $\norm{\bz_{U^c}}_\infty < 1$ (Theorem~\ref{thm:primaldual}) and $A_U^T A_U$ is invertible (corollary of Theorem~\ref{thm:isom}).
\end{enumerate}

In Section~\ref{section:optdist}, we prove Condition (1) by dissecting the subgradient optimality condition. In Section~\ref{section:optdual} we then prove Condition (2) and complete the proof of Proposition~\ref{prop:statistical}. 

\subsection{Computational recovery}
\label{section:techniques:computational}

For Proposition~\ref{prop:algorithm}, we want to solve Program~\ref{eq:program}, i.e. minimize $f(x) = \nll(x;A,y) + \lambda \norm{x}_1.$ The gradient of $\nll(x;A,y)$ doesn't have a closed form, but it can be written cleanly as an expectation:
\[\nabla \nll(x;A,y) = \frac{1}{m} \sum_{j=1}^m A_j^T(\mathbb{E} Z_{A_j,x} - y_j).\]
Let us assume that $Z_{A_j,x}$ can be sampled efficiently. Then we may hope to optimize $f(x)$ by stochastic gradient descent. But problematically, in our high-dimensional setting $f$ is nowhere strongly convex. So while we can apply the following general result from convex optimization, it has several strings attached:

\begin{theorem}\label{thm:psgd-main}
Let $f: \mathbb{R}^n \to \mathbb{R}$ be a convex function achieving its optimum at $\bx \in \mathbb{R}^n$. Let $x^{(0)},x^{(1)},\dots,x^{(N)}$ be a sequence of random vectors in $\mathbb{R}^n$. Suppose that $x^{(i+1)} = x^{(i)} - \eta v^{(i)}$ where $\mathbb{E}[v^{(i)}|x^{(i)}] \in \partial f(x^{(i)})$. Set $\bar{x} = \frac{1}{N}\sum_{i=1}^N x^{(i)}$. Then \[\mathbb{E}[f(\bar{x})] - f(\bx) \leq (\eta N)^{-1} \mathbb{E}\left[\norm{x^{(0)} - \bx}_2^2\right] + \eta N^{-1} \sum_{i=1}^N \mathbb{E}\left[\norm{v^{(i)}}_2^2\right].\]
\end{theorem}

In particular, to apply this result, we need to solve three technical problems:
\begin{enumerate}
\item We need to efficiently find an initial point $x^{(0)}$ with bounded distance from $\bx$.
\item The gradient does not have bounded norm for arbitrary $x \in \mathbb{R}^n$. Thus, we need to pick a projection set in which the bound holds.
\item Since $f$ is not strongly convex, we need to convert the bound on $f(\bar{x}) - f(\bx)$ into a bound on $\norm{\bar{x} - \bx}_2$.
\end{enumerate}
As defined in Section~\ref{section:effalg}, our solution is the projection set $\mathcal{E}_r = \{x \in \mathbb{R}^n: \norm{Ax - y}_2 \leq r\sqrt{m}\}$, for an appropriate constant $r>0$.
To pick an initial point in $\mathcal{E}_r$, we solve the program $x^{(0)} = \argmin_{x \in \mathcal{E}_r} \norm{x}_1.$ This estimate is biased due to the truncation, but the key point is that by classical results from compressed sensing, it has bounded distance from $x^*$ (and therefore from $\bx$).

The algorithm then consists of projected stochastic gradient descent with projection set $\mathcal{E}_r$. To bound the number of update steps required for the algorithm to converge to a good estimate of $\bx$, we need to solve several statistical problems (which are direct consequences of assumptions in Theorem~\ref{thm:psgd-main}). 

\paragraph{Properties of $\mathcal{E}_r$.} First, $\bx$ must be feasible and a bounded distance from the initial point (due to high-dimensionality, $\mathcal{E}_r$ is unbounded, so this is not immediate). The following lemmas formalize this; see Sections~\ref{proof:solution-feasible} and~\ref{proof:initial-bound} for the respective proofs. Lemma~\ref{lemma:solution-feasible} specifies the choice of $r$.

\begin{lemma}\label{lemma:solution-feasible}
With high probability, $\bx \in \mathcal{E}_r$ for an appropriate constant $r > 0$.
\end{lemma}

\begin{lemma}\label{lemma:initial-bound}
With high probability, $\norm{x^{(0)} - \bx}_2 \leq O(1)$.
\end{lemma}

Second, the SGD updates at points within $\mathcal{E}_r$ must be unbiased estimates of the gradient, with bounded square-norm in expectation. The following lemma shows that the updates $v^{(t)}$ defined in Section~\ref{section:effalg} satisfy this property. See Section~\ref{proof:boundedstep} for the proof.

\begin{lemma}\label{lemma:boundedstep}
With high probability over $A$, the following statement holds. Let $0 \leq t < T$. Then $\mathbb{E}[v^{(t)}|x^{(t)}] \in \partial f(x^{(t)})$, and $\mathbb{E}[\norm{v^{(t)}}_2^2] \leq O(n)$.
\end{lemma}

\paragraph{Addressing the lack of strong convexity.} Third, we need to show that the algorithm converges in parameter space and not just in loss. That is, if $f(x) - f(\bx)$ is small, then we want to show that $\norm{x-\bx}_2$ is small as well. Note that $f$ is not strongly convex even in $\mathcal{E}_r$, due to the high dimension. So we need a more careful approach. In the subspace $\mathbb{R}^U$, $f$ is indeed strongly convex near $\bx$, as shown in the following lemma (proof in Section~\ref{proof:subspace-sc}):

\begin{lemma}\label{lemma:subspace-sc}
There is a constant $\zeta$ such that with high probability over $A$, $f(x) - f(\bx) \geq \frac{\zeta}{2} \norm{x - \bx}_2^2$ for all $x \in \mathbb{R}^n$ with $\supp(x) \subseteq U$ and $\norm{x - \bx}_2 \leq 1$.
\end{lemma}

But we need a bound for all $\mathbb{R}^n$. The idea is to prove a lower bound on $f(x)-f(\bx)$ for $x$ near $\bx$, and then use convexity to scale the bound linearly in $\norm{x-\bx}_2$. The above lemma provides a lower bound for $x$ near $\bx$ if $\supp(x) \subseteq U$, and we need to show that adding an $\mathbb{R}^{U^c}$-component to $x$ increases $f$ proportionally. This is precisely the content of Theorem~\ref{thm:slackness}. Putting these pieces together we obtain the following lemma. See Section~\ref{proof:sc} for the full proof.

\begin{lemma}\label{lemma:sc}
There are constants $c_f>0$ and $c'_f$ such that with high probability over $A$ the following holds. Let $x \in \mathbb{R}^n$. If $f(x) - f(\bx) \leq c_f(\log n)/m^3$, then \[f(x) - f(\bx) \geq c'_f \frac{\log n}{m} \norm{x - \bx}_2^2.\]
\end{lemma}




\paragraph{Convergence of PSGD.} It now follows from the above lemmas and Theorem~\ref{thm:psgd-main} that the PSGD algorithm, as outlined above and described in Section~\ref{section:effalg}, converges to a good approximation of $\bx$ in a polynomial number of updates. The following theorem formalizes the guarantee. See Section~\ref{proof:correctness} for the proof.

\begin{theorem}\label{thm:correctness}
With high probability over $A$ and over the execution of the algorithm, we get $$\norm{\bar{x} - \bx}_2 \leq O(\sqrt{(k \log n)/m}).$$
\end{theorem}

\paragraph{Efficient implementation.} Finally, in Section~\ref{appendix:algo-details} we prove that initialization and each update step is efficient. Efficient gradient estimation in the projection set (i.e. sampling $Z_{A_j,x}$) does not follow from the prior work, since our projection set is by necessity laxer than that of the prior work \cite{Daskalakis2018}. So we replace their rejection sampling procedure with a novel approximate sampling procedure under mild assumptions about the truncation set. Together with the convergence bound claimed in Theorem~\ref{thm:correctness}, these prove Proposition~\ref{prop:algorithm}.

\begin{namedproof}{Proposition~\ref{prop:algorithm}}
The correctness guarantee follows from Theorem~\ref{thm:correctness}. For the efficiency guarantee, note that the algorithm performs initialization and then $N = \poly(n)$ update steps. By Section~\ref{appendix:algo-details}, the initialization takes $\poly(n)$ time, and each update step takes $\poly(n) + T(e^{-\Theta(m)})$. This implies the desired bounds on overall time complexity.
\end{namedproof}

\section*{Acknowledgments}

This research was supported by NSF Awards IIS-1741137, CCF-1617730 and CCF-1901292, by a Simons Investigator Award, by the DOE PhILMs project (No. DE-AC05-76RL01830), by the DARPA award HR00111990021, by the MIT Undergraduate Research Opportunities Program, and by a Google PhD Fellowship.


\bibliographystyle{plain}
\bibliography{ref,bib}

\appendix

\section{Bounding solution of restricted program}\label{section:optdist}

In this section we prove that $\norm{\bx - x^*}_2$ is small with high probability, where $\bx$ is a solution to Program~\ref{eq:res-program}. Specifically, we use regularization parameter $\lambda = \Theta(\sqrt{(\log n)/m})$, and prove that $\norm{\bx - x^*}_2 \leq O(\sqrt{(k \log n)/m})$.

The proof is motivated by the following rephrasal of part (a) of Lemma~\ref{lemma:pdw}:
\begin{equation}\label{eq:main}
-\lambda \bz_U = \frac{1}{m} A_U^T(\mathbb{E}[Z_{A,\bx}] - \mathbb{E}[Z_{A,x^*}]) + \frac{1}{m} A_U^T(\mathbb{E}[Z_{A,x^*}] - y)
\end{equation}

where $\norm{\bz_U}_\infty \leq 1$. For intuition, consider the untruncated setting: then $\mathbb{E}[Z_t] = t$, so the equation is simply \[-\lambda \bz_U = \frac{1}{m} A_U^T A_U(\bx_U - x^*_U) - \frac{1}{m} A_U^T w\] where $w \sim N(0,1)^m$. Since $w$ is independent of $A_U^T$ and has norm $\Theta(m)$, each entry of $A_U^T w$ is Gaussian with variance $\Theta(m)$, so $\frac{1}{m}A_U^T w$ has norm $\Theta(\sqrt{k/m})$. Additionally, $\norm{\lambda\bz_U}_2 \leq \lambda \sqrt{k} = O(\sqrt{(k\log n)/m})$. Finally, $\frac{1}{m}A_U^T A_U$ is a $\Theta(1)$-isometry, so we get the desired bound on $\bx_U - x^*_U$.

Returning to the truncated setting, one bound still holds, namely $\norm{\lambda \bz_U}_2 \leq \lambda \sqrt{k}$. The remainder of the above sketch breaks down for two reasons. First, $\mathbb{E}[Z_{A,x^*}] - y$ is no longer independent of $A$. Second, bounding $\frac{1}{m}A_U^T(\mathbb{E}[Z_{A,\bx}] - \mathbb{E}[Z_{A,x^*}])$ no longer implies a bound on $\bx_U - x^*_U$.

The first problem is not so hard to work around; we can still bound $A_U^T(\mathbb{E}[Z_{A,x^*}] - y)$ as follows; see Section~\ref{proof:gradientbound} for the proof.

\begin{lemma}\label{lemma:gradientbound}
With high probability over $A$ and $y$,
$\norm{A_U^T(\mathbb{E}[Z_{A,x^*}] - y)}^2_2 \leq \alpha^{-1} k m \log n.$
\end{lemma}

So in equation~\ref{eq:main}, the last term is $O(\sqrt{(k\log n)/m})$ with high probability. The first term is always $O(\sqrt{(k\log n)/m})$, since $\norm{\bz_U}_2 \leq \sqrt{k}$. So we know that $\frac{1}{m}A_U^T(\mathbb{E}[Z_{A,\bx}] - \mathbb{E}[Z_{A,x^*}])$ has small norm. Unfortunately this does not imply that $\mathbb{E}[Z_{A,\bx}] - \mathbb{E}[Z_{A,x^*}]$ has small norm, but as motivation, assume that we have such a bound.

Since $A_U$ is a $\Theta(\sqrt{m})$-isometry, bounding $\bx - x^*$ is equivalent to bounding $A\bx - Ax^*$. To relate this quantity to $\mathbb{E}[Z_{A,\bx}] - \mathbb{E}[Z_{A,x^*}]$, our approach is to lower bound the derivative of $\mu_t = \mathbb{E}[Z_t]$ with respect to $t$. The derivative turns out to have the following elegant form (proof in Section~\ref{proof:derivative}):

\begin{lemma}\label{lemma:derivative}
For any $t \in \mathbb{R}$, $\frac{d}{dt} \mu_t = \Var(Z_t)$.
\end{lemma}

Crucially, $\Var(Z_t)$ is nonnegative, and relates to survival probability. By integrating a lower bound on the derivative, we get the following lower bound on $\mu_t - \mu_{t^*}$ in terms of $t-t^*$. The bound is linear for small $|t-t^*|$, but flattens out as $|t-t^*|$ grows. See Section~\ref{proof:diff} for the proof.

\begin{lemma}\label{lemma:diff}
Let $t,t^* \in \mathbb{R}$. Then $\sign(\mu_t-\mu_{t^*}) = \sign(t-t^*)$. Additionally, for any constant $\beta>0$ there is a constant $c=c(\beta)>0$ such that if $\gamma_S(t^*) \geq \beta$, then $|\mu_t - \mu_{t^*}| \geq c\min(1, |t - t^*|).$
\end{lemma}

If we want to use this lemma to prove that $\norm{\mathbb{E}[Z_{A,\bx}] - \mathbb{E}[Z_{A,x^*}]}_2$ is at least a constant multiple of $\norm{A(\bx-x^*)}_2$, we face two obstacles: (1) $\gamma_S(A_jx^*)$ may not be large for all $j$, and (2) the lemma only gives linear scaling if $|A_j(\bx-x^*)| = O(1)$: but this is essentially what we're trying to prove!

To deal with obstacle (1), we restrict to the rows $j \in [m]$ for which $\gamma_S(A_jx^*)$ is large. To deal with obstacle (2), we have a two-step proof. In the first step, we use the $\Omega(1)$-lower bound provided by Lemma~\ref{lemma:diff} to show that $\norm{A(\bx-x^*)}_2 = O(\sqrt{m})$ (so that $|A_j(\bx-x^*)| = O(1)$ on average). In the second step, we use this to get linear scaling in Lemma~\ref{lemma:diff}, and complete the proof, showing that $\norm{A(\bx-x^*)}_2 = O(\sqrt{k\log n})$.

Formally, define $I_\text{good}$ to be the set of indices $j \in [m]$ such that $\gamma_S(A_jx^*) \geq \alpha/2$ and $|A_jx^* - A_j\bx|^2 \leq (6/(\alpha m))\norm{Ax^* - A\bx}^2$. In the following lemmas we show that $I_\text{good}$ contains a constant fraction of the indices, so by the isometry properties we retain a constant fraction of $\norm{A(\bx-x^*)}_2$ when restricting to $I_\text{good}$. See Appendices~\ref{proof:goodsize} and~\ref{proof:goodnorm} for the proofs of Lemmas~\ref{lemma:goodsize} and~\ref{lemma:goodnorm} respectively.

\begin{lemma}\label{lemma:goodsize}
With high probability, $|I_\text{good}| \geq (\alpha/6)m$.
\end{lemma}

\begin{lemma}\label{lemma:goodnorm}
For some constant $\epsilon > 0$, we have that with high probability, $\norm{Ax^* - A\bx}_{I_\text{good}}^2 \geq \epsilon \norm{Ax^* - A\bx}^2$.
\end{lemma}

We now prove the weaker, first-step bound on $\norm{A(\bx - x^*)}_2$. But there is one glaring issue we must address: we made a simplifying assumption that $\norm{\mathbb{E}[Z_{A,\bx}] - \mathbb{E}[Z_{A,x^*}]}$ is small. All we actually know is that $\norm{A_U^T(\mathbb{E}[Z_{A,\bx}] - \mathbb{E}[Z_{A,x^*}]}_2$ is small. And $A_U^T$ has a nontrivial null space.

Here is a sketch of how we resolve this issue. Let $a = A(\bx-x^*)$ and $b = \mu_{A\bx} - \mu_{Ax^*}$; we want to show that if $\norm{a}$ is large then $\norm{A_U^Tb}$ is large. Geometrically, $\norm{A_U^Tb}$ is approximately proportional to the distance from $b$ to the subspace $\text{Null}(A_U^T)$. Oversimplifying for clarity, we know that $|b_j| \geq c|a_j|$ for all $j$. This is by itself insufficient. The key observation is that we also know $\sign(a_j) = \sign(b_j)$ for all $j$. Thus, $b$ lies in a hyperoctant shifted to have corner $ca$. Since $ca$ lies in the row space of $A_U^T$, it's perpendicular to $\text{Null}(A_U^T)$, so the closest point to $\text{Null}(A_U^T)$ in the shifted hyperoctant should be $ca$.

Formalizing this geometric intuition yields the last piece of the proofs of the following theorems. See Section~\ref{proof:dist} for the full proofs.

\begin{theorem}\label{thm:weakdist}
There are positive constants $c'_\text{reg} = c'_\text{reg}(\alpha)$, $M' = M'(\alpha)$, and $C' = C'(\alpha)$ with the following property. Suppose that $\lambda \leq c'_\text{reg}/\sqrt{k}$ and $m \geq M'k\log n$. Then with high probability, $\norm{A_Ux^* - A_U\bx}_2 \leq C'\sqrt{m}$.
\end{theorem}

\begin{theorem}\label{thm:dist}
There are positive constants $c''_\text{reg} = c''_\text{reg}(\alpha)$, $M'' = M''(\alpha)$, and $C'' = C''(\alpha)$ with the following property. Suppose that $\lambda \leq c''_\text{reg}/\sqrt{k}$ and $m \geq M''k\log n$. Then $\norm{x^* - \bx}_2 \leq C''(\lambda\sqrt{k} + \sqrt{(k\log n)/m})$ with high probability.
\end{theorem}

\section{Proof of statistical recovery}\label{section:optdual}

Extend $\bz$ to $\mathbb{R}^n$ by defining \[\bz_{U^c} = -\frac{1}{\lambda m} A_{U^c}^T(\mathbb{E} Z_{A,\bx} - y).\]

We would like to show that $\norm{z_{U^c}}_\infty < 1$. Since $A_{U^c}^T$ is independent of $\mathbb{E}[Z_{A,\bx}]-y$, each entry of $A_{U^c}^T(\mathbb{E}[Z_{A,\bx}]-y)$ is Gaussian with standard deviation $\norm{\mathbb{E}[Z_{A,\bx}]-y}_2$. It turns out that a bound of $O(\lambda \sqrt{km} + \sqrt{m})$ suffices. To get this bound, we decompose \[\mathbb{E}[Z_{A,\bx}] - y = A(\bx-x^*) + \mathbb{E} R_{A,\bx} - (y - Ax^*)\] and bound each term separately. Here we are defining $R_t = Z_t - t$, and $R_{a,x} = Z_{a,x} - a^T x$ and $R_{A,x} = Z_{A,x} - Ax$ similarly.

We present the proof of the following lemmas in Section~\ref{proof:skewnorm} and Section~\ref{proof:ynorm} respectively. 

\begin{lemma}\label{lemma:skewnorm}
There is a constant $c = c(\alpha)$ such that under the conditions of Theorem~\ref{thm:dist}, with high probability over $(A,y)$, $\norm{\mathbb{E}[R_{A,\bx}]}_2^2 \leq cm.$
\end{lemma}

\begin{lemma}\label{lemma:ynorm}
There is a constant $c_y = c_y(\alpha)$ such that $\norm{R_{A,x^*}}_2^2 \leq c_ym$ with high probability.
\end{lemma}

Combining the above lemmas with the bound on $\norm{\bx-x^*}_2$ from the previous section, we get the desired theorem. See Section~\ref{proof:primaldual} for the full proof.

\begin{theorem}\label{thm:primaldual}
There are constants $M = M(\alpha)$, $\sigma = \sigma(\alpha)$, and $d = d(\alpha)$ with the following property. Suppose $m \geq Mk\log n$, and $\lambda = \sigma \sqrt{(\log n)/m}$. Then with high probability we have $\norm{\bz_{U^c}}_\infty < 1$.
\end{theorem}

As an aside that we'll use later, this proof can be extended to any random vector near $\bx$ with support contained in $U$ (proof in Section~\ref{proof:slackness}).

\begin{theorem}\label{thm:slackness}
There are constants $M = M(\alpha)$, $\sigma = \sigma(\alpha)$, and $d = d(\alpha)$ with the following property. Suppose $m \geq Mk\log n$ and $\lambda = \sigma\sqrt{(\log n)/m}$. If $X \in \mathbb{R}^n$ is a random variable with $\supp(X) \subseteq U$ always, and $\norm{\bx - X}_2 \leq 1/m$ with high probability, then with high probability $\norm{\frac{1}{m}A_{U^c}(\mathbb{E} Z_{A,X} - y)}_\infty \leq \lambda/2$.
\end{theorem}

Returning to the goal of this section, it remains to show that $A_U^T A_U$ is invertible with high probability. But this follows from the isometry guarantee of Theorem~\ref{thm:isom}. Our main statistical result, Proposition \ref{prop:statistical}, now follows.

\begin{namedproof}{Proposition~\ref{prop:statistical}}
Take $M$, $\sigma$, and $d$ as in the statement of Theorem~\ref{thm:primaldual}. Let $m \geq Mk \log n$ and $\lambda = \sigma \sqrt{(\log n)/m}$. Let $\hat{x} \in \mathbb{R}^n$ be any optimal solution to the regularized program, and let $\bx \in \mathbb{R}^U$ be any solution to the restricted program. By Theorem~\ref{thm:primaldual}, with high probability we have $\norm{x^* - \bx} \leq d\sqrt{(k\log n)/m}$ and $\norm{\bz_{U^c}} < 1$; and by Theorem~\ref{thm:isom}, $A_U^T A_U$ is invertible. So by Lemma~\ref{lemma:pdw}, it follows that $\bx = \hat{x}$. Therefore $\norm{x^* - \hat{x}} \leq d\sqrt{(k \log n)/m}$.
\end{namedproof}

\section{Primal-dual witness method}\label{appendix:pdw}

\begin{namedproof}{Lemma~\ref{lemma:gradient}}
For a single sample $(A_j, y_j)$, the partial derivative in direction $x_i$ is
\begin{align*}
\frac{\partial}{\partial x_i} \nll(x; A_j, y_j)
&= A_{ji}(A_jx - y) + \frac{\frac{\partial}{\partial x_i} \int_S e^{-(A_jx - z)^2/2} \, dz}{\int_S e^{-(A_jx - z)^2/2} \, dz} \\
&= A_{ji}(A_jx - y) - \frac{\int_S A_{ji} (A_jx - z) e^{-(A_jx - z)^2/2} \, dz}{\int_S e^{-(A_jx - z)^2/2} \, dz} \\
&= A_{ji}(A_jx - y) - \mathbb{E}[A_{ji} (A_jx - Z_{A_jx})]
\end{align*}
where expectation is taken over the random variable $Z_{A_jx}$ (for fixed $A_j$). Simplifying yields the expression \[\nabla \nll(x; A_j, y_j) = A_j(\mathbb{E}[Z_{A_jx}] - y).\]
The second partial derivative of $\nll(x;A_j,y_j)$ in directions $x_{i_1}$ and $x_{i_2}$ is therefore
\begin{align*}
\frac{\partial^2}{\partial x_{i_1} \partial x_{i_2}} \nll(x;A_j,y_j)
&= \frac{\partial}{\partial x_{i_1}} A_{ji_2}(\mathbb{E}[Z_{A_jx}] - y) \\
&= A_{ji_2} \frac{\partial}{\partial x_{i_1}}\left(\frac{\int_S z e^{-(A_jx - z)^2/2} \, dz}{\int_S e^{-(A_jx - z)^2/2} \, dz} - y\right) \\
&= A_{ji_2} \Bigg(\frac{\frac{\partial}{\partial x_{i_1}} \int_S ze^{-(A_jx - z)^2/2} \, dz}{\int_S e^{-(A_jx - z)^2/2} \, dz} -\\
&\qquad \frac{\int_S ze^{-(A_jx - z)^2/2} \, dz \frac{\partial}{\partial x_{i_1}} \int_S e^{-(A_jx - z)^2/2} \, dz}{\left(\int_S e^{-(A_jx - z)^2/2} \, dz\right)^2}\Bigg) \\
&= A_{ji_2}(\mathbb{E}[-A_{ji_1}Z_{A_jx}(A_jx - Z_{A_jx})] - \mathbb{E}[Z_{A_jx}]\mathbb{E}[-A_{ji_1}(A_jx - Z_{A_jx})] \\
&= A_{ji_1}A_{ji_2}\Var(Z_{A_jx}).
\end{align*}
We conclude that \[H(x;A_j, y_j) = A_j^TA_j\Var(Z_{A_jx}).\] Averaging over all samples yields the claimed result.
\end{namedproof}

The following lemma collects several useful facts that are needed for the PDW method. Parts (a) and (b) are generically true for any $\ell_1$-regularized convex program; part (c) is a holdover from the untruncated setting that is still true. The proof is essentially due to \cite{Wainwright2009}, although part (c) now requires slightly more work.

\begin{lemma}\label{lemma:sgopt}
Fix any $(A,y)$.
\begin{enumerate}[label = (\alph*)]
\item A vector $x \in \RR^n$ is optimal for Program~\ref{eq:program} if and only if there exists some $z \in \partial \norm{x}_1$ such that \[\nabla \nll(x;A,y) + \lambda z = 0.\]
\item Suppose that $(x,z)$ are as in (a), and furthermore $|z_i| < 1$ for all $i \not \in \supp(x)$. Then necessarily $\supp(\hat{x}) \subseteq \supp(x)$ for any optimal solution $\hat{x}$ to Program~\ref{eq:program}.
\item Suppose that $(x,z)$ are as in (b), with $I = \supp(x)$. If $A_I^T A_I$ is invertible, then $x$ is the unique optimal solution to Program~\ref{eq:program}.
\end{enumerate}
\end{lemma}

\begin{proof}
Part (a) is simply the subgradient optimality condition in a convex program. 

Part (b) is a standard fact about duality; we provide a proof here. Let $\hat{x}$ be any optimal solution to Program~\ref{eq:program}. We claim that $\hat{x}^T z = \norm{\hat{x}}_1$. To see this, first note that $x^T z = \norm{x}_1$, since $x_i z_i = |x_i|$ always holds by definition of a subgradient for the $\ell_1$ norm. Now, by optimality of $x$ and $\hat{x}$, we have $f(x) = f(\hat{x}) \leq f(tx + (1-t)\hat{x})$ for all $0 \leq t \leq 1$. Therefore by convexity, $f(tx + (1-t)\hat{x}) = f(x)$ for all $0 \leq t \leq 1$. Since $f$ is the sum of two convex functions, both must be linear on the line segment between $x$ and $\hat{x}$. Therefore \[\nll(tx + (1-t)\hat{x}) = t\nll(x) + (1-t)\nll(\hat{x})\] for all $0 \leq t \leq 1$. We conclude that \[(\nabla \nll(x)) \cdot (\hat{x} - x) = \nll(\hat{x}) - \nll(x) = \norm{x}_1 - \norm{\hat{x}}_1.\] Since $\nabla \nll(x) + z = 0$ by subgradient optimality, it follows that $z^T(\hat{x} - x) = \norm{\hat{x}}_1 - \norm{x}_1$. Hence, $z^T \hat{x} = \norm{\hat{x}}_1$. Since $|z_i| \leq 1$ for all $i$, if $|z_i| < 1$ for some $i$ then necessarily $\hat{x}_i = 0$ for equality to hold.

For part (c), if $A_I^T A_I$ is invertible, then it is (strictly) positive definite. The Hessian of Program~\ref{eq:res-program} is \[\frac{1}{m} \sum_{j=1}^m A_{I,j}^T A_{I,j} \Var(Z_{A_j, x}).\] Since $\Var(Z_{A_j,x})$ is always positive, there is some $\epsilon > 0$ (not necessarily a constant) such that \[\frac{1}{m} \sum_{j=1}^m A_{I,j}^T A_{I,j} \Var(Z_{A_j, x}) \succcurlyeq \frac{1}{m} \epsilon \sum_{j=1}^m A_{I,j}^T A_{I,j} = \frac{1}{m} \epsilon A_I^T A_I.\]
Thus, the Hessian of the restricted program is positive definite, so the restricted program is strictly convex. Therefore the restricted program has a unique solution. By part (b), any solution to the original program has support in $I$, so the original program also has a unique solution, which must be $x$.
\end{proof}

As with the previous lemma, the following proof is essentially due to \cite{Wainwright2009} (with a different subgradient optimality condition).

\begin{namedproof}{Lemma~\ref{lemma:pdw}}
By part (a) of Lemma~\ref{lemma:sgopt}, a vector $x \in \mathbb{R}^n$ is optimal for Program~\ref{eq:program} if and only if there is some $z \in \partial \norm{x}_1$ such that
\[\frac{1}{m} A^T(\mathbb{E} Z_{A,x} - y) + \lambda z = 0.\]
This vector equality can be written in block form as follows:
\[\frac{1}{m} \begin{bmatrix} A_U^T \\ A_{U^c}^T \end{bmatrix} \left(\mathbb{E} Z_{A,x} - y\right) + \lambda \begin{bmatrix} z_U \\ z_{U^c} \end{bmatrix} = 0.\]
Since $\bx$ is optimal in $\mathbb{R}^U$, there is some $\bz_U \in \partial \norm{\bx}_1$ such that $(\bx,\bz_U)$ satisfy the first of the two block equations. This is precisely part (a). If furthermore $\bx$ is zero-extended to $\mathbb{R}^n$, and $\bz$ is extended as in part (b), and $\bz$ satisfies $\norm{\bz_{U^c}}_\infty \leq 1$, then since $x_i = 0$ for all $i \not \in U$, we have that $\bz$ is a subgradient for $\norm{\bx}_1$. Therefore $\bx$ is optimal for Program~\ref{eq:program}. If $\norm{\bz_{U^c}}_\infty < 1$ and $A_U^TA_U$ is invertible, then $\bx$ is the unique solution to Program~\ref{eq:program} by parts (b) and (c) of Lemma~\ref{lemma:sgopt}.
\end{namedproof}

\section{Sparse recovery from the Restricted Isometry Property}

In this section we restate a theorem due to \cite{Candes2006} about sparse recovery in the presence of noise. Our statement is slightly generalized to allow a trade-off between the isometry constants and the sparsity. That is, as the sparsity $k$ decreases relative to the isometry order $s$, the isometry constants $\tau, T$ are allowed to worsen.

\begin{theorem}[\cite{Candes2006}]\label{thm:crt}
Let $B \in \mathbb{R}^{m \times n}$ be a matrix satisfying the $s$-Restricted Isometry Property \[\tau \norm{v}_2 \leq \norm{Bv}_2 \leq T \norm{v}_2\] for all $s$-sparse $v \in \mathbb{R}^n$. Let $w^* \in \mathbb{R}^n$ be $k$-sparse for some $k < s$, and let $w \in \mathbb{R}^n$ satisfy $\norm{w}_1 \leq \norm{w^*}_1$. Then \[\norm{B(w-w^*)}_2 \geq \left(\tau(1 - \rho) - T\rho\right) \norm{w-w^*}_2\] where $\rho = \sqrt{k/(s-k)}$.
\end{theorem}

\begin{proof}
Let $h = w - w^*$ and let $T_0 = \supp(w^*)$. Then \[\norm{w^*}_1 \geq \norm{w}_1 = \norm{h_{T_0^C}}_1 + \norm{(h + w^*)_{T_0}}_1 \geq \norm{h_{T_0^C}}_1 + \norm{w^*}_1 - \norm{h_{T_0}}_1,\] so $\norm{h_{T_0}}_1 \geq \norm{h_{T_0^C}}_1$. Without loss of generality assume that $T_0^C = \{1,\dots,|T_0^C|\}$, and $|h_i| \geq |h_{i+1}|$ for all $1 \leq i < |T_0^C|$. Divide $T_0^C$ into sets of size $s' = s-k$ respecting this order: \[T_0^C = T_1 \cup T_2 \cup \dots \cup T_r.\] Then the Restricted Isometry Property gives 
\begin{equation} \norm{Bh}_2 
\geq \norm{Bh_{T_0 \cup T_1}}_2 - \sum_{t=2}^r \norm{Bh_{T_t}}_2
\geq \tau \norm{h_{T_0 \cup T_1}}_2 - T \sum_{t=2}^r \norm{h_{T_t}}_2 
\label{eq:ripeq}
\end{equation}
For any $t \geq 1$ and $i \in T_{t+1}$, we have $h_i \leq \norm{h_{T_t}}_1/s'$, so that \[\norm{h_{T_{t+1}}}_2^2 \leq \frac{\norm{h_{T_t}}_1^2}{s'}.\]
Summing over all $t \geq 2$, we get \[\sum_{t=2}^r \norm{h_{T_t}}_2 \leq
\frac{1}{\sqrt{s'}} \sum_{t=1}^r \norm{h_{T_t}}_1 = \frac{\norm{h_{T_0^C}}_1}{\sqrt{s'}} \leq \frac{\norm{h_{T_0}}_1}{\sqrt{s'}} \leq \sqrt{\frac{k}{s'}} \norm{h}_2.\]
The triangle inequality implies that $\norm{h_{T_0\cup T_1}}_2 \geq (1 - \sqrt{k/s'})\norm{h}_2$. Returning to Equation~\ref{eq:ripeq}, it follows that \[\norm{Bh}_2 \geq \left(\tau(1 - \sqrt{k/s'}) - T\sqrt{k/s'}\right)\norm{h}_2\] as claimed.
\end{proof}

\section{Summary of the algorithm} \label{sec:algos}

  \begin{algorithm}[H]
  \caption{Projected Stochastic Gradient Descent.}
  \begin{algorithmic}[1]
    \Procedure{Sgd}{$N, \lambda$}\Comment{\textit{$N$: number of steps, $\lambda$: parameter}}
    \State  $x^{(0)} \gets \argmin \norm{x}_1$ s.t. $x \in \mathcal{E}_r$ \Comment{\textit{see the Appendix \ref{appendix:algo-details} for details}}
    \For{$t = 1, \dots, N$}
      \State $\eta_t \gets \frac{1}{\sqrt{nN}}$
      \State $v^{(t)} \gets \Call{GradientEstimation}{x^{(t - 1)})}$
      \State $w^{(t)} \gets x^{(t - 1)} - \eta_t w^{(t)}$
      \State $x^{(t)} \gets \argmin_{x \in \mathcal{E}_r} \norm{x - w^{(t)}}_2$ \Comment{\textit{see the Appendix \ref{appendix:algo-details} for details}}
    \EndFor\label{euclidendwhile}
    \State \textbf{return} $\bar{x} \gets \frac{1}{N} \sum_{t = 1}^N x^{(t)}$\Comment{\textit{output the average}}
  \EndProcedure
  \end{algorithmic}
  \label{alg:projectedSGD}
  \end{algorithm}
  \vspace{-10pt}
  \begin{algorithm}[H]
  \caption{The function to estimate the gradient of the $\ell_1$ regularized negative log-likelihood.}
  \label{alg:estimateGradient}
  \begin{algorithmic}[1]
    \Function{GradientEstimation}{$x$}
    \State Pick $j$ at random from $[n]$
    \State Use Assumption \ref{assumption:sampling} or Lemma \ref{lem:samplingUnionOfIntervalsLemma} to sample $z \sim Z_{A_jx^{(t)}}$
    \State \textbf{return} $A_j(z - y_j)$
  \EndFunction
  \end{algorithmic}
  \label{alg:rejectionSampling}
  \end{algorithm}

\section{Algorithm details}
\label{appendix:algo-details}

In this section we fill in the missing details about the algorithm's efficiency. Since we have already seen that the algorithm converges in $O(\text{poly}(n))$ update steps, all that remains is to show that the following algorithmic problems can be solved efficiently:
\begin{enumerate}
    \item (Initial point) Compute $x^{(0)} = \argmin_{x \in \mathcal{E}_r} \norm{x}_1.$
    \item (Stochastic gradient) Given $x^{(t)} \in \mathcal{E}_r$ and $j \in [m]$, compute a sample $A_j(z - y_j)$, where $z \sim Z_{A_jx^{(t)}}$.
    \item (Projection) Given $w^{(t)} \in \mathbb{R}^n$, compute $x^{(t+1)} = \argmin_{x \in \mathcal{E}_r} \norm{x - w^{(t)}}_2$.
\end{enumerate}

\paragraph{Initial point.} To obtain the initial point $x^{(0)}$, we need to solve the program

\begin{equation*}
\begin{array}{ll}
\text{minimize} & \norm{x}_1 \\
\text{subject to} & \norm{Ax - y}_2 \leq r \sqrt{m}.
\end{array}
\end{equation*}

This program has come up previously in the compressed sensing literature (see, e.g., \cite{Candes2006}). It can be recast as a Second-Order Cone Program (SOCP) by introducing variables $x^+,x^- \in \mathbb{R}^n$:

\begin{equation*}
\begin{array}{ll}
\text{minimize} & \sum_{i=1}^n (x^+_i - x^-_i) \\
\text{subject to} & \norm{Ax^+ - Ax^- - y}_2 \leq r \sqrt{m}, \\
& x^+ \geq 0, \\
& -x^- \geq 0.
\end{array}
\end{equation*}

Thus, it can be solved in polynomial time by interior-point methods (see \cite{boyd2004convex}).

\paragraph{Stochastic gradient.} In computing an unbiased estimate of the gradient, the only challenge is sampling from $Z_{A_jx^{(t)}}$. By Assumption~\ref{assumption:sampling}, this takes $T(\gamma_S(A_jx^{(t)}))$ time. We know from Lemma~\ref{lemma:total-survival} that $\gamma_S(A_jx^*) \geq \alpha^{2m}$. Since $x^{(t)},x^* \in \mathcal{E}_r$, we have from Lemma~\ref{lemma:survival-decay} that \[\gamma_S(A_jx^{(t)}) \geq \gamma_S(t^*)^2 e^{-|A_j(x^{(t)} - x^*)|^2 - 2} \geq \alpha^{4m} e^{-4r^2m - 2} \geq e^{-\Theta(m/\alpha)}.\]
Thus, the time complexity of computing the stochastic gradient is $T(e^{-\Theta(m/\alpha)})$. 

In the special case when the truncation set $S$ is a union of $r$ intervals, there is a sampling algorithm with time complexity $T(\beta) = \poly(r, \log(1/\beta, n))$ (Lemma~\ref{lem:samplingUnionOfIntervalsLemma}). Hence, in this case the time complexity of computing the stochastic gradient is $\poly(r, n)$.

To be more precise, we instantiate Lemma~\ref{lem:samplingUnionOfIntervalsLemma} with accuracy $\zeta = 1/(nL)$, where $L = \poly(n)$ is the number of update steps performed. This gives some sampling algorithm $\mathcal{A}$. In each step, $\mathcal{A}$'s output distribution is within $\zeta$ of the true distribution $N(t,1;S)$. Consider a hypothetical sampling algorithm $\mathcal{A}'$ in which $\mathcal{A}$ is run, and then the output is altered by rejection to match the true distribution. Alteration occurs with probability $\zeta$. Thus, running the PSGD algorithm with $\mathcal{A}'$, the probability that any alteration occurs is at most $L\zeta = o(1)$. As shown by Theorem~\ref{thm:psgd}, PSGD with $\mathcal{A}'$ succeeds with high probability. Hence, PSGD with $\mathcal{A}$ succeeds with high probability as well.

\paragraph{Projection.} The other problem we need to solve is projection onto set $\mathcal{E}_r$:

\begin{equation*}
\begin{array}{ll}
\text{minimize} & \norm{x - v}_2 \\
\text{subject to} & \norm{Ax - y}_2 \leq r \sqrt{m}.
\end{array}
\end{equation*}

This is a convex QCQP, and therefore solvable in polynomial time by interior point methods (see \cite{boyd2004convex}).

\section{Isometry properties}
\label{appendix:isom}

Let $A \in \mathbb{R}^{m \times n}$ consist of $m$ samples $A_i$ from Process~\ref{process}. In this section we prove the following theorem:

\begin{theorem}\label{thm:isom}
For every $\epsilon > 0$ there are constants $\delta>0$, $M$, $\tau > 0$ and $T$ with the following property. Let $V \subseteq [n]$. Suppose that $m \geq M|V|$. With probability at least $1 - e^{-\delta m}$ over $A$, for every subset $J \subseteq [m]$ with $|J| \geq \epsilon m$, the $|J| \times k$ submatrix $A_{J,V}$ satisfies \[\tau \sqrt{m}\norm{v}_2 \leq \norm{A_{J,V}v}_2 \leq T \sqrt{m}\norm{v}_2 \qquad \forall \, v \in \mathbb{R}^V.\]
\end{theorem}

We start with the upper bound, for which it suffices to take $J = [m]$.

\begin{lemma}\label{lemma:upperisom}
Let $V \subseteq [n]$. Suppose that $m \geq |V|$. There is a constant $T = T(\alpha)$ such that \[\Pr[\smax(A_V) > T] \leq e^{-\Omega(m)}.\]
\end{lemma}

\begin{proof}
In the process for generating $A$, consider the matrix $A'$ obtained by not discarding any of the samples $a \in \mathbb{R}^n$. Then $A'$ is a $\xi \times n$ matrix for a random variable $\xi$; each row of $A'$ is a spherical Gaussian independent of all previous rows, but $\xi$ depends on the rows. Nonetheless, by a Chernoff bound, $\Pr[\xi > 2m/\alpha] \leq e^{-m/(3\alpha)}.$ In this event, $A'$ is a submatrix of $2m/\alpha \times n$ matrix $B$ with i.i.d. Gaussian entries. By \cite{Rudelson2010}, \[\Pr[\smax(B_V) > C\sqrt{2m/\alpha}] \leq e^{-cm}\] for some absolute constants $c, C > 0$. Since $A'$ is a submatrix of $B$ with high probability, and $A$ is a submatrix of $A'$, it follows that \[\Pr[\smax(A_V) > C\sqrt{2m/\alpha}] \leq e^{-\Omega(m)}\] as desired.
\end{proof}

For the lower bound, we use an $\epsilon$-net argument.

\begin{lemma}\label{lemma:dplower}
Let $\epsilon > 0$ and let $v \in \mathbb{R}^n$ with $\norm{v}_2 = 1$. Let $a \sim N(0,1)^n$. Then \[\Pr[|a^Tv| < \alpha \epsilon \sqrt{\pi/2} | a^T x^* + Z \in S] < \epsilon.\]
\end{lemma}

\begin{proof}
From the constant survival probability assumption,
\[\Pr[|a^Tv| < \delta | a^T x^* + Z \in S] \leq \alpha^{-1} \Pr[|a^T v| < \delta].\]
But $a^T v \sim N(0,1)$, so $\Pr[|a^T v| < \delta] \leq 2\delta/\sqrt{2\pi}.$ Taking $\delta = \alpha \epsilon \sqrt{\pi/2}$ yields the desired bound.
\end{proof}

\begin{lemma}\label{lemma:singleisom}
Let $V \subseteq [n]$. Fix $\epsilon > 0$ and fix $v \in \mathbb{R}^V$ with $\norm{v}_2 = 1$. There are positive constants $\tau_0 = \tau_0(\alpha,\epsilon)$ and $c_0 = c_0(\alpha,\epsilon)$ such that \[\Pr\left[\exists J \subseteq [m]: (|J| \geq \epsilon m) \land (\norm{A_{J,V}v}_2 < \tau_0)\right] \leq e^{-c_0m}.\]
\end{lemma}

\begin{proof}
For each $j \in [m]$ let $B_j$ be the indicator random variable for the event that $|A_{j,V} v| < \alpha \epsilon/3$. Let $B = \sum_{j = 1}^m B_j$. By Lemma~\ref{lemma:dplower}, $\mathbb{E} B < \epsilon m/3$. Each $B_j$ is independent, so by a Chernoff bound, \[\Pr[B > \epsilon m/2] \leq e^{-\epsilon m/18}.\]
In the event $[B \leq \epsilon m/2]$, for any $J \subseteq [m]$ with $|J| \geq \epsilon m$ it holds that \[\norm{A_{J,V}v}_2^2 = \sum_{j \in J} (A_{j,V}v)^2 \geq \sum_{j \in J: B_j = 0} (A_{j,V}v)^2 \geq (\alpha \epsilon/3) B \geq \alpha \epsilon^2 m/6.\]
So the event in the lemma statement occurs with probability at most $e^{-\epsilon m/18}.$
\end{proof}

Now we can prove the isometry property claimed in Theorem~\ref{thm:isom}.

\begin{namedproof}{Theorem~\ref{thm:isom}}
Let $V \subseteq [n]$. Let $\epsilon > 0$. Take $\gamma = 4|V|/(c_0m)$, where $c_0 = c_0(\alpha,\epsilon)$ is the constant in the statement of Lemma~\ref{lemma:singleisom}. Let $\mathcal{B} \subseteq \mathbb{R}^V$ be the $k$-dimensional unit ball. Let $\mathcal{D} \subset \mathcal{B}$ be a maximal packing of $(1+\gamma/2)\mathcal{B}$ by radius-$(\gamma/2)$ balls with centers on the unit sphere. By a volume argument, \[|\mathcal{D}| \leq \frac{(1 + \gamma/2)^k}{(\gamma/2)^k} \leq e^{2k/\gamma} \leq e^{c_0 m/2}.\]
Applying Lemma~\ref{lemma:singleisom} to each $v \in \mathcal{D}$ and taking a union bound, \[\Pr[\exists J \subseteq [m], v \in \mathcal{D}: (|J| \geq \epsilon m) \land (\norm{A_{J,V}v}_2 < \tau_0)] \leq e^{-c_0 m/2}.\]
So with probability $1 - e^{-\Omega(m)}$, the complement of this event holds. And by Lemma~\ref{lemma:upperisom}, the event $\smax(A_V) \leq T\sqrt{m}$ holds with probability $1 - e^{-\Omega(m)}$. In these events we claim that the conclusion of the theorem holds. Take any $v \in \mathbb{R}^V$ with $\norm{v}_2 = 1$, and take any $J \subseteq [m]$ with $|J| \geq \epsilon m$. Since $\mathcal{D}$ is maximal, there is some $w \in \mathcal{D}$ with $\norm{v-w}_2 \leq \gamma$. Then \[\norm{A_{J,V}v}_2 \geq \norm{A_{J,V}w}_2 - \norm{A_{J,V}(v-w)}_2 \geq \tau_0 - \gamma T.\]
But $\gamma \leq 4/(c_0 M)$. For sufficiently large $M$, we get $\gamma < \tau_0/(2T)$. Taking $\tau = \tau_0/2$ yields the claimed lower bound.
\end{namedproof}

As a corollary, we get that $A_U^T$ is a $\sqrt{m}$-isometry on its row space up to constants (of course, this holds for any $V \subseteq [n]$ with $|V| = k$, but we only need it for $V = U$).

\begin{corollary}\label{cor:rowisom}
With high probability, for every $u \in \mathbb{R}^k$, \[\frac{\tau^2}{T} \sqrt{m} \norm{A_U u}_2 \leq \norm{A_U^T A_U u}_2 \leq \frac{T^2}{\tau} \sqrt{m} \norm{A_U u}_2.\]
\end{corollary}

\begin{proof}
By Theorem~\ref{thm:isom}, with high probability all eigenvalues of $A_U^T A_U$ lie in the interval $[\tau\sqrt{m},T\sqrt{m}]$. Hence, all eigenvalues of $(A_U^T A_U)^2$ lie in the interval $[\tau^2 m, T^2 m]$. But then
\[\norm{A_U^T A_U u}_2 = u^T (A_U^T A_U)^2 u \geq \tau^2 m u^T u \geq \frac{\tau^2}{T} \sqrt{m} \norm{A_U u}_2.\]
The upper bound is similar.
\end{proof}

We also get a Restricted Isometry Property, by applying Theorem~\ref{thm:isom} to all subsets $V \subseteq [n]$ of a fixed size.

\begin{corollary}[Restricted Isometry Property]\label{corollary:rip}
There is a constant $M$ such that for any $s>0$, if $m \geq Ms\log n$, then with high probability, for every $v \in \mathbb{R}^n$ with $|\supp(v)| \leq s$, \[\tau \sqrt{m} \norm{v}_2 \leq \norm{Av}_2 \leq T\sqrt{m} \norm{v}_2.\]
\end{corollary}

\begin{proof}
We apply Theorem~\ref{thm:isom} to all $V \subseteq [n]$ with $|V| = s$, and take a union bound over the respective failure events. The probability that there exists some set $V \subseteq [n]$ of size $s$ such that the isometry fails is at most \[\binom{n}{s} e^{-\delta m} \leq e^{s\log n - \delta m}.\]
If $m \geq Ms\log n$ for a sufficiently large constant $M$, then this probability is $o(1)$.
\end{proof}

From this corollary, our main result for adversarial noise (Theorem~\ref{thm:mainAdversarial}) follows almost immediately:

\begin{namedproof}{Theorem~\ref{thm:mainAdversarial}}
Let $M'$ be the constant in Corollary~\ref{corollary:rip}. Let $\rho = \min(\tau/(4T), 1/3)$, and let $M = (1+1/\rho^2)M'$. Finally, let $s = (1+1/\rho^2)k$.

Let $\epsilon > 0$. Suppose that $m \geq Mk\log n$ and $\norm{Ax^* - y} \leq \epsilon$. Then $m \geq M's\log n$, so by Corollary~\ref{corollary:rip}, $A/\sqrt{m}$ satisfies the $s$-Restricted Isometry Property.

By definition, $\hat{x}$ satisfies $\norm{A\hat{x} - y}_2 \leq \epsilon$ and $\norm{\hat{x}}_1 \leq \norm{x^*}_1$ (by feasibility of $x^*$). Finally, $x^*$ is $k$-sparse. We conclude from Theorem~\ref{thm:crt} and our choice of $\rho$ that
\[\norm{(A/\sqrt{m})(\hat{x} - x^*)}_2 \geq (\tau(1-\rho) - T\rho)\norm{\hat{x} - x^*}_2 \geq \frac{\tau}{2} \norm{\hat{x} - x^*}_2.\]
But $\norm{A(\hat{x} - x^*)}_2 \leq 2\epsilon$ by the triangle inequality. Thus, $\norm{\hat{x} - x^*}_2 \leq \tau\epsilon/\sqrt{m}.$
\end{namedproof}

\section{Projected Stochastic Gradient Descent}

In this section we present the exact PSGD convergence theorem which we use, together with a proof for completeness.

\begin{theorem}\label{thm:psgd}
Let $f: \mathbb{R}^n \to \mathbb{R}$ be a convex function achieving its optimum at $\bx \in \mathbb{R}^n$. Let $\mathcal{P} \subseteq \mathbb{R}^n$ be a convex set containing $\bx$. Let $x^{(0)} \in \mathcal{P}$ be arbitrary. For $1 \leq t \leq T$ define a random variable $x^{(t)}$ by \[x^{(t)} = \text{Proj}_\mathcal{P}(x^{(t-1)} - \eta v^{(t-1)}),\] where $\mathbb{E}[v^{(t)}|x^{(t)}] \in \partial f(x^{(t)})$ and $\eta$ is fixed. Then \[\mathbb{E}[f(\bar{x})] - f(\bx) \leq (\eta T)^{-1} \mathbb{E}\left[\norm{x^{(0)} - \bx}_2^2\right] + \eta T^{-1} \sum_{i=1}^T \mathbb{E}\left[\norm{v^{(i)}}_2^2\right]\] where $\bar{x} = \frac{1}{T}\sum_{i=1}^T x^{(i)}$.
\end{theorem}

\begin{proof}
Fix $0 \leq t < T$. We can write
\[\norm{x^{(t+1)} - \bx}_2^2 \leq \norm{(x^{(t)} - \eta v^{(t)}) - \bx}_2^2 = \norm{x^{(t)} - \bx}_2^2 - 2\eta \langle v^{(t)}, x^{(t)} - \bx \rangle + \eta^2 \norm{v^{(t)}}_2^2\] since projecting onto $\mathcal{E}_r$ cannot increase the distance to $\bx \in \mathcal{E}_r$.

Taking expectation over $v^{(t)}$ for fixed $x^{(0)},\dots,x^{(k)}$, we have
\begin{align*}
\mathbb{E}\left[\norm{x^{(t+1)} - \bx}_2^2 \middle| x^{(0)},\dots,x^{(k)}\right]
&\leq \norm{x^{(t)} - \bx}_2^2 - 2 \eta \langle \mathbb{E} v^{(t)}, x^{(t)} - \bx \rangle + \eta^2 \mathbb{E}\left[\norm{v^{(t)}}_2^2\right]\\
&\leq \norm{x^{(t)} - \bx}_2^2 - 2 \eta(f(x^{(t)}) - f(\bx)) + \eta^2 \mathbb{E}\left[\norm{v^{(t)}}_2^2\right]
\end{align*}
where the last inequality uses the fact that $\mathbb{E} v^{(t)}$ is a subgradient for $f$ at $x^{(t)}$. Rearranging and taking expectation over $x^{(0)},\dots,x^{(t)}$, we get that \[2\left(\mathbb{E}\left[f(x^{(t)})\right] - f(\bx)\right) \leq \eta^{-1}\left(\mathbb{E}\left[\norm{x^{(t)} - \bx}_2^2\right] - \mathbb{E}\left[\norm{x^{(t+1)} - \bx}_2^2\right]\right) + \eta \mathbb{E}\left[\norm{v^{(t)}}_2^2\right].\]
But now summing over $0 \leq t < T$, the right-hand side of the inequality telescopes, giving
\begin{align*}
\mathbb{E}[f(\bar{x})] - f(\bx) 
&\leq \frac{1}{T} \sum_{t=0}^{T-1} \mathbb{E}[f(x^{(t)})] - f(\bx) \\
&\leq \frac{1}{\eta T} \mathbb{E}\left[\norm{x^{(0)} - \bx}_2^2\right] + \frac{\eta}{T}\sum_{t=0}^{T-1} \mathbb{E}\left[\norm{v^{(t)}}_2^2\right].
\end{align*}
This is the desired bound.
\end{proof}

\section{Survival probability}

In this section we collect useful lemmas about truncated Gaussian random variables and survival probabilities.

\begin{lemma}[\cite{Daskalakis2018}]
\label{lemma:variance-lb}
Let $t \in \mathbb{R}$ and let $S \subset \mathbb{R}$ be a measurable set. Then $\Var(Z_t) \geq C\gamma_S(t)^2$ for a constant $C>0$.
\end{lemma}


\begin{lemma}[\cite{Daskalakis2018}] \label{lemma:survival-decay}
For $t, t^* \in \mathbb{R}$, \[\log \frac{1}{\gamma_S(t)} \leq 2\log \frac{1}{\gamma_S(t^*)} + |t - t^*|^2 + 2.\]
\end{lemma}

\begin{lemma}[\cite{Daskalakis2018}]\label{lemma:rvar}
For $t \in \mathbb{R}$, \[\mathbb{E}[R_t^2] \leq 2 \log \frac{1}{\gamma_S(t)} + 4.\]
\end{lemma}

\begin{lemma}\label{lemma:total-survival}
With high probability, \[\sum_{j=1}^m \log \frac{1}{\gamma_S(A_jx^*)} \leq 2m\log\left(\frac{1}{\alpha}\right).\]
\end{lemma}

\begin{proof}
Let $X_j = \log 1/\gamma_S(A_jx^*)$ for $j \in [m]$, and let $X = X_1+\dots+X_m$. Since $X_1,\dots,X_m$ are independent and identically distributed, \[\mathbb{E}[e^X] = \mathbb{E}[e^{X_j}]^m = \mathbb{E}\left[\frac{1}{\gamma_S(A_jx^*)}\right]^m = \left(\frac{\mathbb{E}_{a \sim N(0,1)^n}[1]}{\mathbb{E}_{a \sim N(0,1)^n}[\gamma_S(a^T x^*)]}\right)^m \leq \alpha^{-m}.\]
Therefore \[\Pr[X > 2m\log 1/\alpha] = \Pr[e^X > e^{2m\log 1/\alpha}] \leq e^{-m\log 1/\alpha}\] by Markov's inequality.
\end{proof}

\section{Omitted proofs}

\subsection{Proof of Lemma~\ref{lemma:gradientbound}}
\label{proof:gradientbound}

We need the following computation:

\begin{lemma}
For any $i \in [n]$ and $j \in [m]$, \[\mathbb{E}_A A_{ji}^2 \Var(Z_{A_j,x^*}) \leq \alpha^{-1}.\]
\end{lemma}

\begin{proof}
By Assumption~\ref{assumption:csp},
\begin{align*}
\mathbb{E}_{A_j}[A_{ji}^2 \Var(Z_{A_j,x^*})] 
&= \frac{\mathbb{E}_{a \sim N(0,1)^n}[\gamma_S(a^T x^*) a_i^2 \Var(Z_{a,x^*})]}{\mathbb{E}_{a \sim N(0,1)^n}[\gamma_S(a^T x^*)]} \\
&\leq \alpha^{-1} \mathbb{E}_{a \sim N(0,1)^n}[\gamma_S(a^T x^*) a_i^2 \Var(Z_{a,x^*})].
\end{align*}
But for any fixed $t \in \mathbb{R}$,
\[\Var(Z_t) \leq \mathbb{E}[(Z_t - t)^2] = \mathbb{E}[Z^2|Z+t \in S] \leq\gamma_S(t)^{-1} \mathbb{E}[Z^2] = \gamma_S(t)^{-1}.\]
Therefore
\[\mathbb{E}_{A_j}[A_{ji}^2 \Var(Z_{A_j,x^*})] \leq \alpha^{-1} \mathbb{E}_{a \sim N(0,1)^n}[a_i^2] \leq \alpha^{-1}\] as desired.
\end{proof}

Now we can prove Lemma~\ref{lemma:gradientbound}.

\begin{namedproof}{Lemma~\ref{lemma:gradientbound}}
We prove that the bound holds in expectation, and then apply Markov's inequality. We need to show that each entry of the vector $A_U^T(\mathbb{E}[Z_{A,x^*}] - Z_{A,x^*}) \in \mathbb{R}^k$ has expected square $O(m)$. A single entry of the vector $A_U^T(\mathbb{E}[Z_{A,x^*}] - Z_{A,x^*})$ is \[\sum_{j=1}^m A_{ji} (Z_{A_j, x^*} - \mathbb{E} Z_{A_j, x^*}),\] so its expected square is \[\mathbb{E} \left(\sum_{j=1}^m A_{ji} (Z_{A_j, x^*} - \mathbb{E} Z_{A_j, x^*})\right)^2.\]
For any $j_1 \neq j_2$, the cross-term is \[\mathbb{E} A_{j_1i}A_{j_2i}(Z_{A_{j_1},x^*} - \mathbb{E} Z_{A_{j_1},x^*})(Z_{A_{j_2},x^*} - \mathbb{E} Z_{A_{j_2},x^*}).\]
But for fixed $A$, the two terms in the product are independent, and they both have mean $0$, so the cross-term is $0$. Thus,
\begin{align*}
\mathbb{E}\left(\sum_{j=1}^m A_{ji} (Z_{A_j, x^*} - \mathbb{E} Z_{A_j, x^*})\right)^2 
&= \sum_{j=1}^m \mathbb{E}_{A, y} A_{ji}^2 (Z_{A_j, x^*} - \mathbb{E} Z_{A_j, x^*})^2 \\
&= \sum_{j=1}^m \mathbb{E}_A A_{ji}^2 \Var(Z_{A_j, x^*}) \\
&\leq \alpha^{-1} m.
\end{align*}
Then \[\mathbb{E}\left[\norm{A_U^T(\mathbb{E}[Z_{A,x^*}] - Z_{A,x^*})}_2^2\right] \leq \alpha^{-1} km.\] Hence with probability at least $1 - 1/\log n$, \[\norm{A_U^T(\mathbb{E}[Z_{A,x^*}] - Z_{A,x^*})}_2^2 \leq \alpha^{-1}km\log n\] by Markov's inequality.
\end{namedproof}

\subsection{Proof of Lemma~\ref{lemma:derivative}
\label{proof:derivative}}

We can write \[\mu_t = \frac{\int_S xe^{-(x-t)^2/2} \, dx}{\int_S e^{-(x-t)^2/2} \, dx}.\]
By the quotient rule, 
\begin{align*}
\frac{d}{dt} \mu_t 
&= -\frac{\int_S x(t-x)e^{-(x-t)^2/2}\, dx}{\int_S e^{-(x-t)^2/2} \, dx} + \frac{\left(\int_S xe^{-(x-t)^2/2}\,dx\right)\left(\int_S (t-x)e^{-(x-t)^2/2}\,dx\right)}{\left(\int_S e^{-(x-t)^2/2} \, dx\right)^2} \\
&= -\mathbb{E}[Z_t(t - Z_t)] + \mathbb{E}[Z_t]\mathbb{E}[t-Z_t] \\
&= \Var(Z_t)
\end{align*}
as desired.

\subsection{Proof of Lemma~\ref{lemma:diff}}
\label{proof:diff}

The fact that $\sign(\mu_t - \mu_{t^*}) = \sign(t - t^*)$ follows immediately from the fact that $\frac{d}{dt} \mu_t = \Var(Z_t) \geq 0$ (Lemma~\ref{lemma:derivative}).

We now prove the second claim of the lemma. Suppose $t^* < t$; the other case is symmetric. Then we have
\[\mu_t - \mu_{t^*} = \int_{t^*}^t \Var(Z_r) \, dr \geq C\int_{t^*}^t \gamma_S(r)^2 \, dr \geq C\beta^2 \int_{0}^{t-t^*} e^{-r^2-2}\, dr\] by Lemmas~\ref{lemma:derivative}, \ref{lemma:variance-lb} and~\ref{lemma:survival-decay} respectively. But we can lower bound
\begin{align*}
\int_0^{t-t^*} e^{-r^2 - 2} \, dr 
&\geq \int_0^{\min(1, t-t^*)} e^{-r^2-2} \, dr \\
&\geq e^{-3} \min(1, t-t^*).
\end{align*}
This bound has the desired form.

\subsection{Proof of Lemma~\ref{lemma:goodsize}}
\label{proof:goodsize}

Since $\mathbb{E}_{a \sim N(0,1)^k} \gamma_S(a^Tx^*) \geq \alpha$ and $\gamma_S(a^Tx^*)$ is always at most $1$, we have $\Pr[\gamma_S(a^Tx^*) \leq \alpha/2] \leq 1 - \alpha/2$. Since the samples are rejection sampled on $\gamma_S(a^T x^*)$, it follows that $\Pr[\gamma_S(A_jx^*) \leq \alpha/2] \leq 1 - \alpha/2$ as well. So by a Chernoff bound, with high probability, the number of $j \in [m]$ such that $\gamma_S(A_jx^*) \leq \alpha/2$ is at most $(1 - \alpha/3)m$.

The condition that $|A_jx^* - A_j\bx|^2 \geq (6/(\alpha m))\norm{Ax^* - A\bx}^2$ is clearly satisfied by at most $(\alpha/6)m$ indices.

\subsection{Proof of Lemma~\ref{lemma:goodnorm}}
\label{proof:goodnorm}

By Lemma~\ref{lemma:goodsize} and Theorem~\ref{thm:isom}, with high probability $A_{I_\text{good},U}$ and $A_U$ both have singular values bounded between $\sqrt{\tau m}$ and $\sqrt{Tm}$ for some positive constants $\tau = \tau(\alpha)$ and $T = T(\alpha)$. In this event, we have \[\norm{A(x^* - \bx)}_{I_\text{good}}^2 \geq \tau m \norm{x^* - \bx}^2 \geq \frac{\tau}{T} \norm{A(x^* - \bx)}^2\] which proves the claim.

\subsection{Proof of Theorems~\ref{thm:weakdist} and~\ref{thm:dist}}
\label{proof:dist}

\begin{namedproof}{Theorem~\ref{thm:weakdist}}
Let $a = A(\bx - x^*)$ and let $b = \mu_{A\bx} - \mu_{Ax^*}$. Our aim is to show that if $\norm{a}_2$ is large, then $\norm{A_U^T b}_2$ is large, which would contradict Equation~\ref{eq:main}. Since $A_U^T$ is not an isometry, we can't simply show that $\norm{b}_2$ is large. Instead, we write an orthogonal decomposition $b = v + A_U u$ for some $u \in \mathbb{R}^k$ and $v \in \mathbb{R}^m$ with $A_U^T v = 0$. We'll show that $\norm{A_U u}_2$ is large. Since $A_U^T b = A_U^T A_U u$, and $A_U^T$ is an isometry on the row space of $A_U$, this suffices.

For every $j \in I_\text{good}$ with $|a_j|>0$, we have by Lemma~\ref{lemma:diff} that \[|b_j| \geq C\min(1, |a_j|) = C|a_j| \min(1/|a_j|, 1)\] where $C$ is the constant which makes Lemma~\ref{lemma:diff} work for indices $j$ with $\gamma_S(A_jx^*) \geq \alpha/2$. Take $C' = \sqrt{6/\alpha}$, and suppose that the theorem's conclusion is false, i.e. $\norm{a}_2 > C'\sqrt{m}$. Also suppose that the events of Lemmas~\ref{lemma:gradientbound} and~\ref{lemma:goodnorm} hold.

Then by the bound $|a_j|^2 \leq (6/(\alpha m))\norm{a}_2^2$ for $j \in I_\text{good}$ we get
\begin{equation}
|b_j| \geq C|a_j| \min\left(\frac{\sqrt{\alpha/6}\sqrt{m}}{\norm{a}_2}, 1\right) = \frac{c\sqrt{m}}{\norm{a}_2} |a_j|
\label{eq:octant}
\end{equation}
where $c = C\sqrt{\alpha/6}$. We assumed earlier that $|a_j|>0$ but Equation~\ref{eq:octant} certainly also holds when $|a_j|=0$.

By Lemma~\ref{lemma:diff}, $a_j$ and $b_j$ have the same sign for all $j \in [m]$. So $a_jb_j \geq 0$ for all $j \in [m]$. Moreover, together with Equation~\ref{eq:octant}, the sign constraint implies that for $j \in I_\text{good}$, \[a_j b_j \geq \frac{c\sqrt{m}}{\norm{a}} a_j^2.\] Summing over $j \in I_\text{good}$ we get \[\sum_{j \in I_\text{good}} a_j^2 \leq \frac{\norm{a}_2}{c\sqrt{m}} \sum_{j \in I_\text{good}} a_j b_j \leq \frac{\norm{a}_2}{c\sqrt{m}} \langle a, b \rangle = \frac{\norm{a}_2}{c\sqrt{m}} \langle a, A_U u \rangle.\]
By Lemma~\ref{lemma:goodnorm} on the LHS and Cauchy-Schwarz on the RHS, we get \[\epsilon \norm{a}_2^2 \leq \frac{\norm{a}_2^2}{c\sqrt{m}} \norm{A_U u}_2.\]
Hence $\norm{A_Uu}_2 \geq \epsilon c \sqrt{m}$. But then $\norm{A_U^T b}_2 = \norm{A_U^T A_U u}_2 \geq (\tau^2/T) \epsilon cm$. On the other hand, Equation~\ref{eq:main} implies that \[\frac{1}{m} \norm{A_U^T b}_2 \leq \lambda\sqrt{k} + \frac{1}{m} \norm{A_U^T(\mathbb{E}[Z_{A,x^*}] - y)}_2 \leq c'_\text{reg} + \sqrt{\alpha^{-1} (k\log n)/m}\] since event (2) holds. This is a contradiction for $M'$ sufficiently large and $c'_\text{reg}$ sufficiently small. So either the assumption $\norm{a}_2 > C'\sqrt{m}$ is false, or the events of Lemma~\ref{lemma:gradientbound} or~\ref{lemma:goodnorm} fail. But the latter two events fail with probability $o(1)$. So $\norm{a}_2 \leq C'\sqrt{m}$ with high probability.

\end{namedproof}

Now that we know that $\norm{x^*-\bx}_2 \leq O(1)$, we can bootstrap to show that $\norm{x^* - \bx}_2 \leq \sqrt{(k \log n)/m}$. While the previous proof relied on the constant regime of the lower bound provided by Lemma~\ref{lemma:diff}, the following proof relies on the linear regime.

\begin{namedproof}{Theorem~\ref{thm:dist}}
As before, let $a = A(\bx - x^*)$ and $b = \mu_{A\bx} - \mu_{Ax^*}$. Suppose that the conclusion of Theorem~\ref{thm:weakdist} holds, i.e. $\norm{a}_2 \leq C'\sqrt{m}$. Also suppose that the events stated in Lemmas~\ref{lemma:gradientbound} and~\ref{lemma:goodnorm} holds. We can make these assumptions with high probability. For $j \in I_\text{good}$, we now know that $|a_j|^2 \leq (6/(\alpha m))\norm{a}_2^2 = O(1)$. Thus, \[|b_j| \geq C |a_j| \cdot \min(1/|a_j|, 1) \geq \delta |a_j|\] where $\delta = C\min(1,\sqrt{\alpha/6}/C')$.
By the same argument as in the proof of Theorem~\ref{thm:weakdist}, except replacing $(c\sqrt{m})/\norm{a}_2$ by $\delta$, we get that \[\epsilon \norm{a}_2^2 \leq \delta^{-1} \norm{a}_2 \cdot \norm{A_U u}_2.\]

Thus, $\norm{a}_2 \leq \epsilon^{-1} \delta^{-1} \norm{A_U u}_2.$ By the isometry property of $A_U^T$ on its row space  (Corollary~\ref{cor:rowisom}), we get \[\norm{a}_2 \leq \frac{\tau^2}{T\epsilon \delta \sqrt{m}} \norm{A_U^T A_U u}_2 = \frac{c'}{\sqrt{m}} \norm{A_U^T b}_2\] for an appropriate constant $c'$. Since $a = A(\bx - x^*)$ and $A_U$ is a $\sqrt{m}$-isometry up to constants (Theorem~\ref{thm:isom}), we get \[\norm{\bx - x^*}_2 \leq \frac{\norm{a}_2}{\tau} \leq \frac{c'}{\tau m} \norm{A_U^T b}_2.\]
By Equation~\ref{eq:main} and bounds on the other terms of Equation~\ref{eq:main}, the RHS of this inequality is $O(\lambda\sqrt{k} + \sqrt{(k \log n)/m})$.
\end{namedproof}

\subsection{Proof of Lemma~\ref{lemma:skewnorm}}
\label{proof:skewnorm}

For $1 \leq j \leq m$ we have by Lemma~\ref{lemma:rvar} that \[(\mathbb{E} R_{A_j,\bx})^2 \leq \mathbb{E}[R_{A_j,\bx}^2] \leq 2 \log \frac{1}{\gamma_S(A_j \bx)} + 4.\]
By Lemma~\ref{lemma:survival-decay}, we have \[\log \frac{1}{\gamma_S(A_j \bx)} \leq 2\log \frac{1}{\gamma_S(A_jx^*)} + |A_j\bx - A_jx^*|^2 + 2.\]
Therefore summing over all $j \in [m]$,
\[\norm{\mathbb{E}R_{A,\bx}}_2^2 \leq 4\sum_{j=1}^m \log \frac{1}{\gamma_S(A_jx^*)} + 2\norm{A(\bx - x^*)}_2^2 + 8m.\]
Lemma~\ref{lemma:total-survival} bounds the first term. Theorems~\ref{thm:dist} and~\ref{thm:isom} bound the second: with high probability, \[\norm{A(\bx-x^*)}_2 \leq 2T(\lambda\sqrt{km} + C''\sqrt{k\log n}).\] Thus, \[\norm{\mathbb{E}[R_{A,\bx}]}_2^2 \leq 8m\log(1/\alpha) + 8m + 8T^2\lambda^2 km + 8(TC'')^2k\log n\] with high probability. Under the assumptions $\lambda \leq c''_\text{reg}/\sqrt{k}$ and $m \geq M''k\log n$, this quantity is $O(m)$.

\subsection{Proof of Lemma~\ref{lemma:ynorm}}
\label{proof:ynorm}

Draw $m$ samples from the distribution $R_{A_j,x^*}$ as follows: pick $a \sim N(0,1)^n$ and $\eta \sim N(0,1)$. Keep sample $\eta$ if $a^T x^* + \eta \in S$; otherwise reject. We want to bound $\eta_1^2 + \dots + \eta_m^2$. Now consider the following revised process: keep all the samples, but stop only once $m$ samples satisfy $a^T x^* + \eta \in S$. Let $t$ be the (random) stopping point; then the random variable $\eta_1^2+\dots+\eta_t^2$ defined by the new process stochastically dominates the random variable $\eta_1+\dots+\eta_m^2$ defined by the original process.

But in the new process, each $\eta_i$ is Gaussian and independent of $\eta_1,\dots,\eta_{i-1}$. With high probability, $t \leq 2m/\alpha$ by a Chernoff bound. And if $\eta'_1,\dots,\eta'_{2m/\alpha} \sim N(0,1)$ are independent then \[{\eta'}_1^2 + \dots + {\eta'}_{2m/\alpha}^2 \leq 4m/\alpha\] with high probability, by concentration of norms of Gaussian vectors. Therefore $\eta_1^2+\dots+\eta_t^2 \leq 4m/\alpha$ with high probability as well.

\subsection{Proof of Theorem~\ref{thm:primaldual}}
\label{proof:primaldual}

Set $\sigma = 4(\sqrt{c}+\sqrt{c_y})$, where $c$ and $c_y$ are the constants in Lemmas~\ref{lemma:skewnorm} and~\ref{lemma:ynorm}. Set $M = \max(16T^2C''^2(\sigma+1)^2/\sigma^2, \sigma^2/{c''}_\text{reg}^2)$. Note that $M$ is chosen sufficiently large that $\lambda = \sigma\sqrt{(\log n)/m} \leq c''_\text{reg}/\sqrt{k}$. 

By Theorem~\ref{thm:dist}, we have with high probability that the following event holds, which we call $E_\text{close}$: \[\norm{x^* - \bx}_2 \leq C''(\lambda\sqrt{k} + \sqrt{(k \log n)/m}) = C''(\sigma + 1)\sqrt{\frac{k\log n}{m}}.\] Now notice that \[\mathbb{E} Z_{A,\bx} - y = A(\bx - x^*) + \mathbb{E} R_{A,\bx} - (y - Ax^*).\]
If $E_\text{close}$ holds, then by Theorem~\ref{thm:isom}, we get $\norm{A(\bx - x^*)}_2 \leq TC''(\sigma+1)\sqrt{k \log n}$. By Lemma~\ref{lemma:skewnorm}, with high probability $\norm{\mathbb{E} R_{A,\bx}}_2 \leq \sqrt{cm}$. And by Lemma~\ref{lemma:ynorm}, with high probability $\norm{y - A x^*}_2 \leq \sqrt{c_ym}$. Therefore
\[\norm{\mathbb{E}Z_{A,\bx} - y}_2 \leq TC''(\sigma+1)\sqrt{k \log n} + \sqrt{c_ym} + \sqrt{cm} \leq \frac{\sigma}{2} \sqrt{m}\] where the last inequality is by choice of $M$ and $\sigma$. Thus, the event \[E: \norm{\mathbb{E} Z_{A,\bx} - y}_2 \leq \frac{\sigma}{2}\sqrt{m}\] occurs with high probability.

Suppose that event $E$ occurs. Now note that $A_{U^c}^T$ has independent Gaussian entries. Fix any $i \in U_c$; since $(A^T)_i$ is independent of $A_U$, $\bx$, and $y$, the dot product \[(A^T)_i(\mathbb{E} Z_{A,\bx} - y)\] is Gaussian with variance $\norm{\mathbb{E} Z_{A,\bx} - y}_2^2 \leq \sigma^2 m/4$. Hence, $\bz_i = \frac{1}{\lambda m} (A^T)_i (\mathbb{E} Z_{A,\bx} - y)$ is Gaussian with variance at most $(\sigma^2 m/4)/(\lambda m)^2 = 1/(4\log n)$. So \[\Pr[|\bz_i| \geq 1] \leq 2e^{-2\log n} \leq \frac{2}{n^2}.\]
By a union bound, \[\Pr[\norm{\bz_{U^c}}_\infty \geq 1] \leq \frac{2}{n}.\]
So the event $\norm{\bz_{U^c}} < 1$ holds with high probability.

\subsection{Proof of Theorem~\ref{thm:slackness}}\label{proof:slackness}

We know from Theorem~\ref{thm:primaldual} that $\norm{\frac{1}{m}A_{U^c}^T(\mathbb{E}[Z_{A,\bx}] - y)}_\infty \leq \lambda/3$. So it suffices to show that \[\frac{1}{m}\norm{A_{U^c}^T(\mathbb{E}[Z_{A,X}] - y) - A_{U^c}^T(\mathbb{E}[Z_{A,\bx}] - y)}_\infty \leq \frac{\lambda}{6}.\]

Thus, we need to show that \[\frac{1}{m}|(A^T)_i(\mathbb{E}[Z_{A,X}] - \mathbb{E}[Z_{A,\bx}])| \leq \frac{\lambda}{6}\] for all $i \in U^c$. Fix one such $i$. Then by Lemma~\ref{lemma:derivative},
\begin{align*}
\norm{\mathbb{E}[Z_{A,X}] - \mathbb{E}[Z_{A,\bx}]}_2^2
&= \sum_{i=1}^m (\mu_{A_iX} - \mu_{A_i\bx})^2 \\
&= \sum_{i=1}^m \left(\int_{A_iX}^{A_i\bx} \Var(Z_t) \, dt\right)^2 \\
&\leq \sum_{i=1}^m (A_iX - A_i\bx)^2 \cdot \sup_{t \in [A_iX,A_i\bx]} \Var(Z_t).
\end{align*} 
By Lemma~\ref{lemma:total-survival}, we have \[\sum_{j=1}^m \log \frac{1}{\gamma_S(A_jx^*)} \leq 2m\log(1/\alpha)\] with high probability over $A$. Assume that this inequality holds, and assume that $\norm{X - \bx}_2 \leq 1$ and $\norm{\bx - x^*}_2 \leq 1$, so that $\norm{X - x^*}_2 \leq 2$. Then by Theorem~\ref{thm:isom}, $\norm{A(X-x^*)}_2 \leq 2T\sqrt{m}$. By Lemma~\ref{lemma:survival-decay}, for every $j \in [m]$ and every $t \in [A_jX, A_j\bx]$,
\[\log \frac{1}{\gamma_S(t)} \leq 2\log \frac{1}{\gamma_S(A_jx^*)} + |t - A_jx^*|^2 +2 \leq cm\] for a constant $c$. Hence, by Lemma~\ref{lemma:rvar}, \[\Var(Z_t) \leq \mathbb{E}[(Z_t - t)^2] \leq 2\log \frac{1}{\gamma_S(t)} + 4 \leq 2cm+4.\]
We conclude that \[\norm{\mathbb{E}[Z_{A,X}] - \mathbb{E}[Z_{A,\bx}]}_2^2 \leq (2cm+4)\norm{AX - A\bx}_2^2 \leq (2cm+4)\frac{T^2m}{m^2} \leq O(1).\]
Additionally, $\norm{(A^T)_i}_2 \leq T\sqrt{m}$ with high probability. Thus, Cauchy-Schwarz entails that \[\frac{1}{m}|(A^T)_i(\mathbb{E}[Z_{A,X}] - \mathbb{E}[Z_{A,\bx}])| \leq \frac{1}{m} T\sqrt{m} \cdot O(1) \leq \frac{\lambda}{6}\] for large $n$.

\subsection{Proof of Lemma~\ref{lemma:solution-feasible}}
\label{proof:solution-feasible}

Note that \[\norm{A\bx - y}_2 \leq \norm{A(\bx - x^*)}_2 + \norm{Ax^* - y}_2.\]
With high probability, $\norm{\bx-x^*}_2 \leq 1$. Theorem~\ref{thm:isom} gives that $\norm{A(\bx-x^*)}_2 \leq T\sqrt{m}$. Furthermore, $\norm{Ax^* - y}_2 \leq 2\sqrt{m/\alpha}$ by Lemma~\ref{lemma:ynorm}.

\subsection{Proof of Lemma~\ref{lemma:initial-bound}}
\label{proof:initial-bound}

With high probability $\bx \in \mathcal{E}_r$ by the above lemma. Note that $\norm{x^{(0)}}_1 \leq \norm{\bx}_1$, and $\norm{A(x^{(0)} - \bx)}_2 \leq 2r\sqrt{m}$. Set $\rho = \min(\tau/(4T), 1/3)$ and $s = k(1 + 1/\rho^2)$. If $m \geq Mk\log n$ for a sufficiently large constant $M$, then by Corollary~\ref{corollary:rip}, $A/\sqrt{m}$ with high probability satisfies the $s$-Restricted Isometry Property. Then by Theorem~\ref{thm:crt} (due to \cite{Candes2006}, but reproduced here for completeness), it follows that $\norm{x^{(0)} - \bx}_2 \leq O(1)$.

\subsection{Proof of Lemma~\ref{lemma:boundedstep}}
\label{proof:boundedstep}

Note that $\sign(x^{(t)})$ is a subgradient for $\norm{x}_1$ at $x = x^{(t)}$. Furthermore, for fixed $A$, 
\[\mathbb{E}\left[A_j(z^{(t)} - y_j)\middle|x^{(t)}\right] = \frac{1}{m}\sum_{j'=1}^m A_{j'}(\mathbb{E}Z_{A_{j'},x^{(t)}} - y_{j'}) = \nabla \nll(x^{(t)}; A, y).\]
It follows that \[\mathbb{E}[v^{(t)}|x^{(t)}] = \mathbb{E}\left[A_j(z^{(t)} - y_j)\middle|x^{(t)}\right] + \sign(x^{(t)})\] is a subgradient for $f(x)$ at $x = x^{(t)}$.

We proceed to bounding $\mathbb{E}[\norm{v^{(t)}}_2^2|x^{(t)}]$. By definition of $v^{(t)}$, \[\norm{v^{(t)}}_2^2 \leq 2\norm{A_j(z^{(t)} - y_j)}_2^2 + 2\norm{\lambda \cdot \sign(x^{(t)})}_2^2\] where $j \in [m]$ is uniformly random, and $z^{(t)}|x^{(t)} \sim Z_{A_j,x^{(t)}}$. Since $\norm{\lambda \cdot \sign(x^{(t)})}_2^2 = o(n)$ it remains to bound the other term. We have that \[\mathbb{E}[\norm{A_j(z^{(t)} - y_j)}_2^2|x^{(t)}] = \frac{1}{m}\sum_{j'=1}^m \mathbb{E}[\norm{A_{j'}(Z_{A_{j'},x^{(t)}} - y_{j'})}_2^2].\]

With high probability, $\norm{A_i}_2^2 \leq 2n$ for all $i \in [m]$. Thus, \[\mathbb{E}[\norm{A_j(z^{(t)} - y_j)}_2^2|x^{(t)}] \leq \frac{n}{m}\sum_{j'=1}^m \mathbb{E}[(Z_{A_{j'},x^{(t)}} - y_{j'})^2].\]

Now \[\sum_{i=1}^m \mathbb{E}[(Z_{A_i,x^{(t)}} - y_i)^2] \leq 2\sum_{i=1}^m (A_ix^{(t)} - y_i)^2 + 2\sum_{i=1}^m \mathbb{E}[(A_ix^{(t)} - Z_{A_i,x^{(t)}})^2].\]
The first term is bounded by $2r^2m$ since $x^{(t)} \in \mathcal{E}_r$. Additionally, \[\norm{A(x^{(t)}-x^*)}_2 \leq \norm{Ax^{(t)} - y}_2 + \norm{Ax^* - y}_2 \leq 2r\sqrt{m}\] since $x^{(t)},x^* \in \mathcal{E}_r$. Therefore the second term is bounded as
\begin{align*}
2\sum_{i=1}^m \mathbb{E}[R_{A_i,x^{(t)}}^2] 
&\leq 4\sum_{i=1}^m \log\left(\frac{1}{\gamma_S(A_ix^{(t)})}\right) + 8m \\
&\leq 8\sum_{i=1}^m \log\left(\frac{1}{\gamma_S(A_ix^*)}\right) + 4\norm{A(x^{(t)}-x^*)}_2^2 + 16m \\
&\leq 64\log(1/\alpha)m + 16r^2m + 80m.
\end{align*}
where the first and second inequalities are by Lemmas~\ref{lemma:rvar} and Lemma~\ref{lemma:survival-decay}, and the third inequality is by Lemma~\ref{lemma:total-survival}. Putting together the two bounds, we get \[\sum_{i=1}^m \mathbb{E}[(Z_{A_i,x^{(t)}} - y_i)^2] \leq O(m),\] from which we conclude that $\mathbb{E}[\norm{v^{(t)}}_2^2|x^{(t)}] \leq O(n)$. The law of total expectation implies that $\mathbb{E}[\norm{v^{(t)}}_2^2] \leq O(n)$ as well.

\subsection{Proof of Lemma~\ref{lemma:subspace-sc}}
\label{proof:subspace-sc}

We need to show that $f_{\mathbb{R}^U}$ is $\zeta$-strongly convex near $\bx$. Since $\norm{x}_1$ is convex, it suffices to show that $\nll(x;A,y)_{\mathbb{R}^U}$ is $\zeta$-strongly convex near $\bx$. The Hessian of $\nll(x;A,y)|_{\mathbb{R}^U}$ is \[H_U(x;A,y) = \frac{1}{m} \sum_{j=1}^m A_{j,U}^T A_{j,U} \Var(Z_{A_j,x}).\] Hence, it suffices to show that \[\frac{1}{m} \sum_{j=1}^m A_{j,U}^T A_{j,U} \Var(Z_{A_j,x}) \succeq \zeta I\] for all $x \in \mathbb{R}^n$ with $\supp(x) \subseteq U$ and $\norm{x - \bx}_2 \leq 1$. Call this region $\mathcal{B}$. With high probability over $A$ we can deduce the following. 

\textbf{(i)} By Theorem~\ref{thm:dist}, we have $\norm{\bx - x^*}_2 \leq d\sqrt{(k\log n)/m}$. As $\norm{x - \bx}_2 \leq 1$ for all $x \in \mathcal{B}$, we get $\norm{A(\bx - x)}_2^2 \leq T^2(d+1)^2m$ for all $x \in \mathcal{B}$.

\textbf{(ii)} By the proof of Lemma~\ref{lemma:goodsize}, the number of $j \in [m]$ such that $\gamma_S(A_jx^*) \leq \alpha/2$ is at most $(1-\alpha/3)m$.

Fix $x \in \mathcal{B}$, and define $J_x \subseteq [m]$ to be the set of indices \[J_x = \{j \in [m]: \gamma_S(A_jx^*) \geq \alpha/2 \land |A_j(x - x^*)|^2 \leq (6/\alpha)T^2(d+1)^2.\}\]

For any $j \in J_x$, \[\log \frac{1}{\gamma_S(A_jx)} \leq 2\log \frac{1}{\gamma_S(A_jx^*)} + |A_j(x - x^*)|^2 + 2 \leq \log(2/\alpha) + (6/\alpha)T^2(d+1)^2 + 2.\] Thus, \[\Var(Z_{A_j,x}) \geq C\gamma_S(A_jx)^2 \geq e^{-\log(2/\alpha) - (6/\alpha)T^2(d+1)^2 - 2} = \Omega(1).\]
Let $\delta$ denote this lower bound---a positive constant. By \textbf{(i)} and \textbf{(ii)}, $|J_x| \geq (\alpha/6)m$, so by Theorem~\ref{thm:isom}, \[H_U(x;A,y) = \frac{1}{m} \sum_{j=1}^m A_{j,U}^T A_{j,U} \Var(Z_{A_j,x}) \succeq \frac{\delta}{m} A_{J_x,U}^T A_{J_x,U} \succeq \delta \tau I\] as desired.

\subsection{Proof of Lemma~\ref{lemma:sc}}
\label{proof:sc}

Let $t = \norm{(x - \bx)_U}_2$. Define $w = \bx + (x - \bx)\min(t^{-1}/m,1)$. Also define $w' = [w_U; 0_{U^c}] \in \mathbb{R}^n$. Then $\norm{(w - \bx)_U}_2 \leq 1/m$, so \[\norm{(\nabla \nll(w'; A,y))_{U^c}}_\infty \leq \frac{\lambda}{2}.\]
Therefore $w_i \cdot (\nabla \nll(w';A,y))_i \leq (\lambda/2)|w_i|$ for all $i \in U^c$, so
\begin{align*}
f(w) - f(w')
&= (\nll(w;A,y) - \nll(w';A,y)) + \lambda(\norm{w}_1 - \norm{w'}_1) \\
&\geq (w - w') \cdot \nabla \nll(w'; A,y) + \lambda\norm{w_{U^c}}_1 \\
&\geq \frac{\lambda}{2} \norm{w_{U^c}}_1.
\end{align*}
Additionally, since $\norm{w' - \bx}_2 \leq 1$ and $\supp(w') \subseteq U$, we know that \[f(w') - f(\bx) \geq \frac{\zeta}{2} \norm{w' - \bx}_2^2.\]
Adding the second inequality to the square of the first inequality, and lower bounding the $\ell_1$ norm by $\ell_2$ norm,
\begin{align*}
\frac{1}{2}(f(w) - f(\bx))^2 + \frac{1}{2}(f(w) - f(\bx)) 
&\geq \frac{1}{2}(f(w) - f(w'))^2 + \frac{1}{2}(f(w') - f(\bx)) \\
&\geq \frac{\lambda^2}{8} \norm{w_{U^c}}_2^2 + \frac{\zeta}{4} \norm{w' - \bx}_2^2 \\
&\geq \frac{\lambda^2}{8} \norm{(w - \bx)_{U^c}}_2^2 + \frac{\zeta}{4} \norm{(w - \bx)_U}_2^2 \\
&\geq \min\left(\frac{\lambda^2}{8}, \frac{\zeta}{4}\right) \norm{w - \bx}_2^2
\end{align*}
Since $f(x) - f(\bx) \leq 1$, by convexity $f(w) - f(\bx) \leq 1$ as well. Hence,
\begin{equation} f(w) - f(\bx) \geq \frac{1}{2}(f(w) - f(\bx))^2 + \frac{1}{2}(f(w) - f(\bx)) \geq \min\left(\frac{\lambda^2}{8}, \frac{\zeta}{4}\right) \norm{w - \bx}_2^2. \label{eq:local-convexity}\end{equation}
We distinguish two cases:
\begin{enumerate}
\item If $t \leq 1/m$, then $w = x$, and it follows from Equation~\ref{eq:local-convexity} that \[f(x) - f(\bx) \geq \min\left(\frac{\lambda^2}{8},\frac{\zeta}{4}\right)\norm{x - \bx}_2^2\] as desired. 

\item If $t \geq 1/m$, then $\norm{(w-\bx)_U}_2 = 1/m$, and thus $\norm{w - \bx}_2 \geq 1/m$. By convexity and this bound, \[f(x) - f(\bx) \geq f(w) - f(\bx) \geq \min\left(\frac{\lambda^2}{8}, \frac{\zeta}{4}\right) \frac{1}{m^2},\] which contradicts the lemma's assumption for a sufficiently small constant $c_f>0$.
\end{enumerate}

\subsection{Proof of Theorem~\ref{thm:correctness}}\label{proof:correctness}

By Lemmas~\ref{lemma:solution-feasible}, \ref{lemma:initial-bound}, and~\ref{lemma:boundedstep}, we are guaranteed that $\bx \in \mathcal{E}_r$, $\norm{x^{(0)} - \bx}_2^2 \leq O(1)$, and $\mathbb{E}[\norm{v^{(t)}}_2^2] \leq O(n)$ for all $t$. Thus, applying Theorem~\ref{thm:psgd} with projection set $\mathcal{E}_r$, step count $T = m^6n$, and step size $\eta = 1/\sqrt{Tn}$ gives $\mathbb{E}[f(\bar{x})] - f(\bx) \leq O(1/m^3)$. Since $f(\bar{x})-f(\bx)$ is nonnegative, Markov's inequality gives \[\Pr[f(\bar{x}) - f(\bx) \leq c_f(\log n)/m^3] \geq 1 - o(1).\] From Theorem~\ref{lemma:sc} we conclude that $\norm{\bar{x} - \bx}_2 \leq O(1/m)$ with high probability.

\section{Efficient sampling for union of intervals} \label{sec:unionIntervalsSampling}

  In this section, in Lemma \ref{lem:samplingUnionOfIntervalsLemma}, we see
that when $S = \cup_{i = 1}^r [a_i, b_i]$, with $a_i, b_i \in \mathbb{R}$, 
then Assumption \ref{assumption:sampling} holds with 
$T(\gamma_S(t)) = \poly(\log(1/\gamma_S(t)), r)$. The only difference is 
that instead of exact sampling we have approximate sampling, but the
approximation error is exponentially small in total variation distance and
hence it cannot affect any algorithm that runs in polynomial time.

\begin{definition}[\textsc{Evaluation Oracle}] \label{def:evaluationOracle}
    Let $f : \mathbb{R} \to \mathbb{R}$ be an arbitrary real function. 
  We define the \textit{evaluation oracle} $\mathcal{E}_f$ of $f$ as an
  oracle that given a number $x \in \mathbb{R}$ and a target accuracy 
  $\eta$ computes an $\eta$-approximate value of $f(x)$, that is
  $\left|\mathcal{E}_f(x) - f(x)\right| \le \eta$.
\end{definition}
  
\begin{lemma} \label{lem:inversionLemma}
    Let $f : \mathbb{R} \to \mathbb{R}_+$ be a real increasing and 
  differentiable function and $\mathcal{E}_f(x)$ an evaluation oracle of
  $f$. Let $\ell \le f'(x) \le L$ for some $\ell, L \in \mathbb{R}_+$. Then we can construct an algorithm that implements the evaluation
  oracle of $f^{-1}$, i.e. $\mathcal{E}_{f^{-1}}$. This implementation on
  input $y \in \mathbb{R}_+$ and input accuracy $\eta$ runs in time
  $T$ and uses at most $T$ calls to the evaluation oracle $\mathcal{E}_{f}$
  with inputs $x$ with representation length $T$ and input accuracy
  $\eta' = \eta/\ell$, with 
  $T = \poly\log(\max\{|f(0)/y|, |y/f(0)|\}, L, 1/\ell, 1/\eta)$.
\end{lemma}

\begin{proof}[Proof of Lemma \ref{lem:inversionLemma}]
    Given a value $y \in \mathbb{R}_+$ our goal is to find an 
  $x \in \mathbb{R}$ such that $f(x) = y$. Using doubling we can find
  two numbers $a, b$ such that $f(a) \le y - \eta'$ and 
  $f(b) \ge y + \eta'$ for some $\eta'$ to be determined later. Because of 
  the lower bound $\ell$ on the derivative of $f$ we have that this step 
  will take $\log((1/\ell) \cdot \max\{|f(0)/y|, |y/f(0)|\})$ steps.
  Then we can use binary search in the interval $[a, b]$ where in each step
  we make a call to the oracle $\mathcal{E}_f$ with accuracy $\eta'$ and we
  can find a point $\hat{x}$ such that 
  $\left| f(x) - f(\hat{x}) \right| \le \eta'$. Because of the upper bound
  on the derivative of $f$ we have that $f$ is $L$-Lipschitz and hence this
  binary search will need $\log(L/\eta')$ time and oracle calls. Now 
  using the mean value theorem we get that for some $\xi \in [a, b]$ it
  holds that 
  $\left| f(x) - f(\hat{x}) \right| = |f'(\xi)| \left|x - \hat{x} \right|$
  which implies that $\left|x - \hat{x}\right| \le \eta'/\ell$, so if we 
  set $\eta' = \ell \cdot \eta$, the lemma follows.
\end{proof}

  Using the Lemma \ref{lem:inversionLemma} and the Proposition 3 of 
\cite{Chevillard12} it is easy to prove the following lemma.

\begin{lemma} \label{lem:samplingOneIntervalLemma}
    Let $[a, b]$ be a closed interval and $\mu \in \mathbb{R}$ such that
  $\gamma_{[a, b]}(\mu) = \alpha$. Then there exists an algorithm that runs
  in time $\poly\log(1/\alpha, \zeta)$ and returns a sample of a
  distribution $\mathcal{D}$, such that 
  $d_{\mathrm{TV}}(\mathcal{D}, N(\mu, 1; [a, b])) \le \zeta$.
\end{lemma}

\begin{proof}[Proof Sketch]
    The sampling algorithm follows the steps:
  (1) from the cumulative distribution function $F$ of the distribution
  $N(\mu, 1; [a, b])$ define a map from $[a, b]$ to $[0, 1]$, (2)
  sample uniformly a number $y$ in $[0, 1]$ (3) using an evaluation oracle 
  for the error function, as per Proposition 3 in \cite{Chevillard12}, 
  and using Lemma \ref{lem:inversionLemma} compute with exponential 
  accuracy the value $F^{-1}(y)$ and return this as the desired sample.
\end{proof}

  Finally using again Proposition 3 in \cite{Chevillard12} and Lemma
\ref{lem:samplingOneIntervalLemma} we can get the following lemma.

\begin{lemma} \label{lem:samplingUnionOfIntervalsLemma}
    Let $[a_1, b_1]$, $[a_2, b_2]$, $\dots$, $[a_r, b_r]$ be closed
  intervals and $\mu \in \mathbb{R}$ such that 
  $\gamma_{\cup_{i = 1}^r [a_i, b_i]}(\mu) = \alpha$. Then there exists
  an algorithm that runs in time $\poly(\log(1/\alpha, \zeta), r)$ and
  returns a sample of a distribution $\mathcal{D}$, such that 
  $d_{\mathrm{TV}}(\mathcal{D}, N(\mu, 1; \cup_{i = 1}^r [a_i, b_i])) \le \zeta$.
\end{lemma}

\begin{proof}[Proof Sketch]
    Using Proposition 3 in \cite{Chevillard12} we can compute
  $\hat{\alpha}_i$ which estimated with exponential accuracy the mass
  $\alpha_i = \gamma_{[a_i, b_i]}(\mu)$ of every interval $[a_i, b_i]$.
  If $\hat{\alpha}_i/\alpha \le \zeta/3r$ then do not consider interval
  $i$ in the next step. From the remaining intervals we can choose one
  proportionally to $\hat{\alpha}_i$. Because of the high accuracy in the
  computation of $\hat{\alpha}_i$ this is $\zeta/3$ close in total 
  variation distance to choosing an interval proportionally to $\alpha_i$.
  Finally after choosing an interval $i$ we can run the algorithm of Lemma
  \ref{lem:samplingOneIntervalLemma} with accuracy $\zeta/3$ and hence the
  overall total variation distance from 
  $N(\mu, 1; \cup_{i = 1}^r [a_i, b_i])$ is at most $\zeta$.
\end{proof}

\end{document}